\documentclass{article}

\usepackage[english]{babel}

\usepackage[letterpaper,top=2cm,bottom=2cm,left=3cm,right=3cm,marginparwidth=1.75cm]{geometry}

\usepackage{amsmath, amsfonts, amssymb, amsthm}
\usepackage{graphicx}
\usepackage{xcolor}
\usepackage[colorlinks=true, allcolors=blue]{hyperref}
\usepackage{bbm}
\usepackage{natbib}
\usepackage[noend]{algorithmic}
\usepackage[ruled,vlined]{algorithm2e}
\usepackage{thm-restate}
\usepackage{comment}
\usepackage{authblk}

\DeclareMathOperator*{\argmax}{argmax}
\DeclareMathOperator*{\argmin}{argmin}
\DeclareMathOperator*{\E}{\mathbb{E}}
\renewcommand{\d}{\mathcal{D}}
\newcommand{\g}{\mathcal{G}}
\newcommand{\p}{\mathcal{P}}
\newcommand{\R}{\mathbb{R}}
\newcommand{\x}{\mathcal{X}}
\newcommand{\y}{\mathcal{Y}}
\newcommand{\z}{\mathcal{Z}}
\newcommand{\1}{\mathbbm{1}}
\newcommand{\cL}{\mathcal{L}}
\newcommand{\cA}{\mathcal{A}}
\newcommand{\ls}{\textrm{LS}}
\newcommand{\range}{\textrm{Range}}

\newtheorem{definition}{Definition}
\newtheorem{theorem}{Theorem}
\newtheorem{lemma}{Lemma}
\newtheorem{claim}{Claim}

\newtheorem{assumption}{Assumption}
\newtheorem{observation}{Observation}
\newtheorem{remark}{Remark}
\newtheorem{fact}{Fact}

\title{The Scope of Multicalibration: \\ Characterizing Multicalibration via Property Elicitation}
\author[1]{Georgy Noarov\thanks{This work was done in part while the author was visiting the Simons Institute for the Theory of Computing.}}
\author[1]{Aaron Roth}
\affil[1]{Department of Computer and Information Sciences, University of Pennsylvania}

\begin{document}
\maketitle

\begin{abstract}
We make a connection between multicalibration and property elicitation and show that (under mild technical conditions) it is possible to produce a multi-calibrated predictor for a continuous scalar property $\Gamma$ if and only if $\Gamma$ is \emph{elicitable}. 

On the negative side, we show that for non-elicitable continuous properties there exist simple data distributions on which even the true distributional predictor is not calibrated. On the positive side, for elicitable $\Gamma$, we give simple canonical algorithms for the batch and the online adversarial setting, that learn a $\Gamma$-multicalibrated predictor. This generalizes past work on multicalibrated means and quantiles, and in fact strengthens existing online  quantile multicalibration results.

To further counter-weigh our negative result, we show that if a property $\Gamma^1$ is not elicitable by itself, but \emph{is} elicitable \emph{conditionally} on another elicitable property $\Gamma^0$, then there is a canonical algorithm that \emph{jointly} multicalibrates $\Gamma^1$ and $\Gamma^0$; this generalizes past work on mean-moment multicalibration. 

Finally, as applications of our theory, we provide novel algorithmic and impossibility results for fair (multicalibrated) risk assessment.
\end{abstract}
\thispagestyle{empty} \setcounter{page}{0}
\clearpage

\section{Introduction}

Consider a distribution $D$ over a labeled data domain $\z =\x\times \R$ of examples with observable features $x \in \x$ and labels $y \in \R$. A predictor $f:\x\rightarrow \R$ is \emph{(mean) calibrated} if, informally, it correctly estimates the \emph{mean label value} even conditional on its own predictions: i.e., $\E_{(x,y) \sim D}[y | f(x) = v] = v$ for all predictions $v$. Calibration is a desirable property, but a weak one, since it only refers to the \emph{average} value of the label, averaged over all examples such that $f(x) = v$; it might be, for example, that there is a structured subset of examples $G \subset \x$ such that $f$ systematically under-estimates label means for examples $x \in G$ --- such a predictor can still be calibrated if it compensates by over-estimating the mean labels for $x \not\in G$.

Multicalibration was introduced by \cite{hebert2018multicalibration} to strengthen the notion of calibration. A multicalibrated predictor is parameterized by a collection of groups $\g \subseteq 2^\x$, and is calibrated not just overall, but also conditional on membership in $G$ for all groups $G \in \g$. That is, for all $v, G$, we must have:
\[\E_{(x,y) \sim D}[y | f(x) = v, x\in G] = v.\]

Multicalibration was generalized from means to moments by \cite{jung_moment}. Thought-provokingly, by way of an explicit counterexample,~\cite{jung_moment} showed that the variance (and other higher moments) cannot be multicalibrated by themselves but \emph{can} be multicalibrated \emph{jointly} with the mean, i.e., as part of a (mean, moment) pair. 
 Later, \cite{gupta2022online} and \cite{jung2022batch} showed how to obtain a \emph{quantile} analogue of multicalibration, which requires that for any target coverage level $\tau \in [0,1]$, for any $v$ in the range of a predictor $f$ and for any $G \in \g$: 
 \[\E[\mathbbm{1}[y \leq f(x)] | f(x) = v, x \in G] = \tau.\]

Thus, by now we have efficient batch \citep{hebert2018multicalibration,jung_moment,jung2022batch} and online \citep{gupta2022online, bastani2022practical} multicalibration algorithms for several natural distributional properties (means and quantiles), an impossibility result for the variance and higher moments, and a result showing how to multicalibrate means and moments together despite  moments not being multicalibratable on their own \citep{jung_moment}. 
But are these one-off results, or is there a more general theory of multicalibration for \emph{distributional properties} --- i.e., arbitrary functionals $\Gamma: \p \to \R$ mapping any distribution $P \in \p$ to a scalar statistic $\Gamma(P)$? To study this question, it is natural to define (cf.~\cite{dwork2022beyond}) that a predictor $f$ is \emph{$(\g,\Gamma)$-multicalibrated} for property $\Gamma$ if for all groups $G \in \g$ and values $\gamma$ in $f$'s range, it holds that:
\[\Gamma(D|(f(x) = \gamma, x \in G)) = \gamma,\]
where $D|(f(x) = \gamma, x \in G)$ is the data distribution conditioned on the event $\{x: f(x)=\gamma, x \in G\}$. 
(When $\g = \{\x\}$, we would simply say that $f$ is $\Gamma$-calibrated.)

\paragraph{Our motivation} In developing our theory of property multicalibration, we are guided by trying to answer several questions of special significance:

\begin{enumerate}
    \item \emph{Which} distributional properties of interest \emph{are} possible to multicalibrate, and which ones are not?
    \item For those properties that \emph{are} possible to multicalibrate in the batch setting, do we always also have a solution in the online adversarial setting, or is there an online-offline separation?
    \item Both in the batch and in the online setting, can one formulate a natural \emph{canonical} algorithm with simple and generic performance guarantees, which takes in a description (in some simple format) of any multicalibratable property of interest and outputs a multicalibrated predictor for it? 
    \item For practically important properties that are not multicalibratable \emph{per se}, are there any reasonably general techniques that would come to the rescue and let us nonetheless achieve some modified notion of multicalibration? (Cf. the case of variance, where packaging it together with multicalibrated means brings it back into the realm of multicalibratable properties as shown by \cite{jung_moment}.)
\end{enumerate}

\subsection{Our results}

We give an (almost) complete answer to all these questions by connecting property multicalibration to the well-studied  theory of \emph{property elicitation}. 

The modern formulation of property elicitation theory is due to~\cite{lambert2008eliciting}, but it has been extensively developed in both earlier and later works. In a nutshell, properties are called \emph{elicitable} if their value on any data distribution can always be directly learned as the minimizer of some loss function over the dataset. For example, means and quantiles are elicitable as they can be solved for, respectively, via least squares and quantile regression. But, for instance,  variance is not elicitable, hence why one typically first predicts the mean and only then computes the variance around it.

An equivalent (subject to mild assumptions) notion is that of \emph{identifiable} properties: an identification function (typically a first-order condition on the property's ``loss function'') tells us if we over- or under-estimated the property value, in expectation over the dataset. For example, an \emph{expected} identification function for a mean predictor $f_m$ is simply $\E_{(x, y)}[f_m(x) - y]$, and an \emph{expected} identification function for a $\tau$-quantile predictor $f_\tau$ is $\Pr_{(x, y)}[y \leq f_\tau(x)] - \tau$, the average overcoverage of $f_\tau$. We give the following collection of results.

\begin{enumerate}
    \item \textbf{A Feasibility Criterion: $\Gamma$-multicalibration is possible if and only if $\Gamma$ is elicitable.} We provide a very general \emph{if-and-only-if} characterization, that subject to mild assumptions categorizes various distributional properties of interest as possible or not possible to multicalibrate. 
    We show that a property $\Gamma$ is \emph{sensible for calibration} (see Definition~\ref{def:sensible}) if and only if it is elicitable, if and only if it is identifiable. See Theorem~\ref{thm:sensible} in Section~\ref{sec:sensibility}. A crucial tool that we use is a central result of~\cite{steinwart_elicitation}: (under mild conditions) a property $\Gamma$ is elicitable $\iff$ it is identifiable $\iff$ its level sets are convex; the key is that we tightly relate sensibility for $\Gamma$-calibration to the convexity of $\Gamma$'s level sets.
    \item \textbf{General canonical batch and online algorithms.} We identify two ``canonical'' $\Gamma$-multicalibration algorithms for elicitable properties $\Gamma$: the batch Algorithm~\ref{alg:batch} and the online adversarial Algorithm~\ref{alg:online}, and prove convergence guarantees for them. See Sections~\ref{sec:batch} and~\ref{sec:sequential}, respectively.
    Our batch algorithm naturally extends the known methods for means and quantiles \citep{hebert2018multicalibration,jung2022batch}. 
    Our online algorithm generalizes  existing online algorithms for means and quantiles~\citep{gupta2022online,bastani2022practical, lee2022online} (and achieves stronger results than prior work for quantiles).

    \item \textbf{\emph{Joint} multicalibration for conditionally elicitable properties.} We show that if a property $\Gamma^0$ is elicitable, and $\Gamma^1$ is \emph{conditionally} elicitable given $\Gamma^1$ (meaning, informally, that conditional on knowing the exact value of $\Gamma^0$, there \emph{is} a regression procedure to learn $\Gamma^1$)
    then (under technical conditions), the pair $(\Gamma^0,\Gamma^1)$ is \emph{jointly} multicalibratable using a canonical Algorithm~\ref{alg:conditional}.
    This generalizes (mean, moment) multicalibration of~\cite{jung_moment}. See Section~\ref{sec:multi}.
    \item \textbf{Applications: Positive and negative results on fair risk assessment.} Prior work on mean, joint mean-moment and quantile multicalibration offer  theoretical and empirical evidence for the promise of deploying multicalibration for the purposes of risk estimation. Previously, however, nothing was known about multicalibrating any of the large collection  of other \emph{risk measures} beyond quantiles and variances. In Section~\ref{sec:examples}, we begin to fill this gap by applying our theory to derive results about a host of risk measures of central significance in financial risk assessment. For example, we show a general negative result that the large family of \emph{distortion risk measures} are not multicalibratable, except for means and quantiles (and two other technical quantile variants). On the positive side, we establish that so-called \emph{Bayes risks} are multicalibratable jointly with the elicitable property whose risk they measure. This is exemplified by \emph{Conditional Value at Risk} (CVaR), also known as \emph{Expected Shortfall} (ES) --- a risk assessment measure of central theoretical and practical significance --- which, as we show, is not multicalibratable on its own but is multicalibratable jointly with quantiles.
\end{enumerate}

\subsection{Additional related work}
The main conceptual contribution of our paper is to connect the literature on \emph{multicalibration} with the literature on \emph{property elicitation}, both of which have a number of related threads. 

\paragraph{Multicalibration} (Mean) multicalibration in the batch setting was introduced by \cite{hebert2018multicalibration}, and has subsequently been generalized in a number of ways. As already discussed, \cite{jung_moment} study mean conditioned moment multicalibration in the batch setting, \cite{gupta2022online} study mean, quantile, and mean conditioned moment multicalibration in the sequential setting, and \cite{jung2022batch} studies quantile multicalibration in the batch setting. These generalizations can be seen as asking for calibration with respect to different distributional properties --- i.e.\ they are generalizations of the type that we characterize in our paper. \cite{dwork2021outcome} study outcome indistinguishability, generalizing batch  multi-calibration in a binary-label setting to allow for \emph{distinguishers} that can evaluate the expectations of arbitrary predicates of triples $(x, f(x), y)$, and consider strengthenings in which the distinguisher gets further access to $f$ --- e.g.\ by being able to query it in arbitrary points, or by having access to its code. In particular, they show that without additional access to $f$, outcome indistinguishability can be reduced to (mean) multicalibration. Note that many distributional properties (e.g.\ variances, quantiles, etc) only become interesting when the label space becomes real valued rather than binary valued.  

The most closely related works study abstract generalizations of multi-calibration. \cite{dwork2022beyond} study a generalization of outcome indistinguishability to real valued labels, and along the way consider batch multicalibration with respect to \emph{linearizing statistics}. In our language, these are distributional properties $\Gamma$ that behave linearly over mixture distributions --- informally that for any two distributions $D_1,D_2$ and any $\alpha \in [0,1]$, $\Gamma(\alpha D_1 + (1-\alpha) D_2) = \alpha \Gamma(D_1) + (1-\alpha) \Gamma(D_2)$. We adopt one of their definitions of multicalibration with respect to general distributional properties. All linearizing statistics have convex level sets (and so are elicitable), but not all elicitable properties are linear in this sense --- for example, quantiles do not linearize. So our characterization of multicalibration implies that it is possible to multicalibrate with respect to a broader class of properties than are studied by \cite{dwork2022beyond}. \cite{lee2022online} study a  general online learning problem that they call ``Online Minimax Multiobjective Optimization'', and derive algorithms for multicalibration along with a number of other applications in this framework. We make use of this framework to derive our sequential multicalibration bounds, using an \emph{identification function} that arises from the connection we make to property elicitation. Recent work of \cite{happymap} studies a one-dimensional generalization of multicalibration that asks for the condition that $\E[c(f(x),x)s(f(x),y)] = 0$ for abstract functions $c$ and $s$, and derives sufficient (but not necessary) conditions under which this can be achieved in the batch setting. The algorithms we derive for batch multicalibration are similar to theirs, where an identification function takes the place of their $s$ function, and a scoring function takes the place of their potential function; they do not consider sequential or multi-dimensional problems. Relative to this line of work, our result is the first to provide a characterization of when property multicalibration can be obtained, and to provide unifying results for both batch and sequential multicalibration.

There are also generalizations in orthogonal directions. \cite{gopalan2022low} define ``low degree multicalibration'', which is a hierarchy of properties of predictors that are still trying to predict \emph{means}, but relax the conditioning event that $f(x) = v$. At the bottom of the hierarchy is \emph{multi-accuracy} \cite{hebert2018multicalibration,kim2019multiaccuracy} which does not condition on $f(x)$ at all. Multicalibration lies at the top of the hierarchy; in between are conditions that depend on $f(x)$ only smoothly, through a degree $k$ polynomial. \cite{gopalan2022low} show that intermediate levels of this hierarchy have some of the desirable properties of multicalibration and can be easier to obtain. Several works \citep{kim2019multiaccuracy,gopalan2022omnipredictors,kim2022universal,globus2023multicalibration} study generalizations of mean multicalibration in which ``groups'' representing subsets of the data domain are relaxed to arbitrary real valued functions and give a number of applications. In this setting, \cite{globus2023multicalibration} provide a characterization of when mean multicalibration implies Bayes optimality. See \cite{rothuncertain} for an introductory exposition of much of this work. 

\paragraph{Property elicitation} \cite{brier1950verification}, \cite{good1992rational}, and \cite{savage1971elicitation} study proper scoring rules, which are contracts for paying experts as a function of their predictions and realized outcomes that maximally reward them (in expectation) if they report the true probability of the outcome event. Since scoring rules directly elicit probabilities, they could in principle be used to elicit an entire probability distribution by eliciting the probability of every event in its support, but this is generally infeasible. Instead, \cite{lambert2008eliciting} introduce the problem of \emph{property elicitation}, whose goal is to design contracts that incentivize experts to truthfully report some \emph{property} of a large or infinite support distribution --- like its mean, variance, median, etc. Informally a property is \emph{elicitable} if there exists some function of a report and an outcome that in expectation over the outcome is minimized at the property value. There is now a large literature on property elicitation; we make use of several key results. \cite{osband1985providing} and \cite{gneiting2011making} define the notion of an identification function for a property, which like a scoring rule is a function of a report and an outcome; an identification function takes value 0 in expectation over the outcome if the report is equal to the property value. \cite{steinwart_elicitation} prove a central characterization theorem (subject to mild technical conditions) --- a continuous property is elicitable if and only if it has an identification function if and only if it has convex level sets. The characterization of \cite{steinwart_elicitation} holds generally for continuous outcome spaces. When the outcome space is finite, \cite{finocchiaro2018convex} show that (subject to technical conditions), elicitable properties can be elicited with convex scoring rules.

\section{Preliminaries}
\label{sec:prelim}

We study prediction problems over labeled datapoints in $\z = \x \times \y$, where $\x$ is a space of feature vectors and $\y \subseteq \R$ is a space of labels. We study both batch (offline) and sequential (online) prediction problems. The online setting will be discussed in Section~\ref{sec:sequential}, and we defer sequential prediction definitions and preliminaries to that section. Informally, property multicalibration is defined identically in the sequential setting as we define it here in the batch setting, with the \emph{empirical} distribution over the realized data sequence taking the place of the underlying dataset distribution considered here. 

In the batch setting, there is a distribution $D \in \Delta \z$ over labeled examples. Such a distribution induces a distribution $X$ over features and $Y$ over labels. Keeping $D$ implicit, we let $Y_x := (Y | \{X = x\})$ denote the conditional label distribution given feature vector $x \in \x$, 
and more generally, given a subset $G \subseteq \x$, we write $Y_G := (Y | \{x \in G\})$ for the conditional label distribution conditional on $x \in G$.
A \emph{predictor} or \emph{model} is simply a real valued mapping $f:\x\rightarrow \mathbb{R}$.

Both offline and online, our goal is to make \emph{(multi)calibrated} predictions about some \emph{distributional property}. We now introduce the requisite background for both these notions.

\subsection{Distributional properties}

A one-dimensional \emph{(distributional) property} is a functional $\Gamma: \p \to \R$, where $\p$ is some space of probability distributions of real-valued random variables. We write $\range_\Gamma$ to denote the range of $\Gamma$, and without loss of generality, scale properties in this paper so that $\range_\Gamma \subseteq [0,1]$. 

Examples of distributional properties include the mean, or the median, or a $\tau$-quantile, or the variance of a distribution (for further examples see Section~\ref{sec:examples}). What all of these notions have in common is that each of them puts a single real number in correspondence with a given distribution.

\paragraph{Formally defining $\p$} Throughout this paper, the basic assumption we make about the space of distributions $\p$ that $\Gamma$ is defined on is that $\p$ is a \emph{convex} subspace of some vector space, so that taking convex combinations over distributions in $\p$ is a well-defined operation. 

The distributions $P \in \p$ will be defined over the label space $\y$, and we have two convenient ways to define the structure of $\p$, depending on whether or not $\y$ is a finite set. If $\y$ is finite ($|\y| = d < \infty)$, then every distribution $P \in \p$ can be embedded into $\R^d$ as an element of the $d$-dimensional simplex $\Delta(d)$, so we simply define that $\p \subseteq \Delta(d) \subset \R^d$ is a convex subset of the simplex $\Delta(d)$, equipped with the usual Euclidean norm. 

If $\y$ is infinite, we assume that $\p \subseteq W_\text{TV}$, where $W_\text{TV}$ is a \emph{Banach space} (i.e.\ a complete normed vector space; see~\cite{diestel1977vector} for an accessible introduction to Banach spaces) of probability distributions over $\y$ that have almost everywhere bounded densities, equipped with the \emph{total variation norm} $||\cdot||_\text{TV}$. In particular, $W_\text{TV}$ is a metric space where the distance between any two probability distributions $P_1, P_2 \in \p$ is computed as the \emph{total variation distance} $||P_1 - P_2||_\text{TV}$.

An assumption we will often place on $\Gamma$ is continuity: i.e.\ $\Gamma$ cannot take drastically different values on very similar distributions. (A great many properties, including means, variances, quantiles, entropies, etc.\ are indeed continuous.)
\begin{definition}[Continuous Property]
    If the functional $\Gamma: \p \to \R$ is continuous (with respect to the metric topology on $\p$ and the standard topology on $\R$), we call $\Gamma$ a \emph{continuous property}.
\end{definition}

\paragraph{Property prediction on datasets} In our batch setting, we deal with datasets over $\x \times \y$, and we are interested in training predictors $f_\Gamma: \x \to \R$ for properties $\Gamma$ of the conditional label distributions $Y_x$ over $\y$. Informally, a good predictor for property $\Gamma$ would satisfy $f_\Gamma(x) \approx \Gamma(Y_x)$ for every $x \in \x$.

Since we study properties $\Gamma: \p \to \R$ defined over some family $\p$ of distributions over the dataset's label space $\y$, we will want to restrict our attention to those dataset distributions $D \in \Delta (\x \times \y)$ whose induced label distributions $Y_x$ belong to $\p$ for all $x \in \x$, so that at least the property $\Gamma$ has a well-defined value on every distribution conditional on any $x \in \x$. We formalize this as follows.

\begin{definition}[$\p$-Compatible Dataset Distribution]
    Given a family of distributions $\p \subseteq \Delta \y$ over the label space $\y$, we say that a dataset distribution $D \in \Delta (\x \times \y)$ over $\x \times \y$ is \emph{$\p$-compatible} if for all $x \in \x$, the induced label distribution given $x$ belongs to $\p$, i.e.\ we have $Y_x \in \p$.
\end{definition}

\subsection{Property elicitation and identification}

We are now ready to formally define three related concepts that are the subject of study in the property elicitation literature. These concepts are: elicitability, identifiability, and level set convexity (CxLS) of distributional properties. All three of them are desirable to have for any property of interest, and are tightly related to each other --- in fact, as we discuss shortly, they are equivalent under some technical assumptions on the distributional property.

\paragraph{Elicitability}
Simply put, a property defined on a family of distributions $\p$ is called \emph{elicitable} if its value on any distribution $P \in \p$ can be obtained by minimizing some loss function in expectation over samples from $P$ --- or, said in the language of statistical learning, by solving a regression problem. As is customary in the elicitation literature, we refer to such loss functions as \emph{scoring functions}: mathematically, a scoring function is just a function $S:\R\times \y\rightarrow \R$.
\begin{definition}[Strictly Consistent Scoring Function]Fix a space of probability distributions $\p$. 
A scoring function $S:\R\times \y\rightarrow \R$ is \emph{strictly $\p$-consistent} for property $\Gamma: \p \to \R$ if:
\[\Gamma(P) = \argmin_{\gamma \in \R} \E_{y \sim P}[S(\gamma,y)] \quad \text{for all $P \in \p$}.\]
We also say that $S$ \emph{elicits} $\Gamma$.
\end{definition}
For brevity, we denote: $S(\gamma,P) = \E_{\gamma \sim P}[S(\gamma, y)]$.

$S$ is said to be \emph{$\p$-order sensitive} for $\Gamma$ if for all $P \in \p$ and $\gamma_1,\gamma_2 \in \R$ such that $|\gamma_1 - \Gamma(P)| < |\gamma_2 - \Gamma(P)|$, it holds that $S(\gamma_1,P) < S(\gamma_2,P).$

\begin{definition}[Elicitable Property]
Fix a space of probability distributions $\p$. A property $\Gamma:\p\rightarrow \R$ is said to be \emph{elicitable} if it has a strictly $\p$-consistent scoring function. 
\end{definition}
As a basic example of the above definitions, the scoring function defined as $S(\gamma, y) = (\gamma - y)^2$ elicits distributional means; thus, means are an elicitable property.

However, a key takeaway from the elicitation literature is that elicitability should not be taken for granted. For instance, the \emph{variance} of a distribution cannot be directly elicited as a minimizer of some loss, without first estimating e.g.\ the mean; and thus the variance is not an elicitable property. We will discuss more such examples in Section~\ref{sec:examples}.

\paragraph{Convexity of level sets (CxLS)}
A simple but deep \emph{necessary} condition for elicitability, first observed by \cite{osband1985providing}, is that elicitable properties must have \emph{convex level sets} (also referred to as \emph{CxLS}). This will be key to our characterization of sensibility for calibration via elicitability.
 The claim is simple but conceptually important so we include the proof for completeness.
\begin{fact}[\cite{osband1985providing}]
    \label{fact:osband}
    Let $\p$ be a convex space of probability distributions, and $\Gamma:\p\rightarrow \R$ be an elicitable property. Then for all $\gamma \in \range_\Gamma$, the \emph{level set} $\{P \in \p : \Gamma(P) = \gamma\}$ is \emph{convex}: for any $P_1, P_2$ with $\Gamma(P_1) = \Gamma(P_2) = \gamma$, $\Gamma(\lambda P_1 + (1-\lambda) P_2) = \gamma$ for all $\lambda \in [0,1]$.
\end{fact}
 \begin{proof}
 Fix any two distributions $P_1,P_2 \in \p$ such that $\Gamma(P_1) = \Gamma(P_2) = \gamma$ and let $S$ be a strictly $\p$-consistent scoring function for $\Gamma$. Since $S$ is a strictly consistent scoring function, we have that $\gamma = \arg\min_{\gamma'} S(\gamma',P_1) = \arg\min_{\gamma'} S(\gamma',P_2)$. For any $\alpha \in [0,1]$ consider $\hat P = \alpha P_1 + (1-\alpha) P_2$. By the convexity of $\p$ we have $\hat P \in \p$, and by linearity of expectation, for any $\gamma' \in \R$ we have:
 \begin{eqnarray*}
 S(\gamma, \hat P) &=& \alpha S(\gamma,P_1) + (1-\alpha) S(\gamma, P_2) \\
 &<& \alpha S(\gamma',P_1) + (1-\alpha) S(\gamma', P_2) \\
 &=& S(\gamma', \hat P)
 \end{eqnarray*}
 Hence $\gamma = \arg\min_{\gamma'} S(\gamma', \hat P)$ and thus $\Gamma(\hat P) = \gamma$, proving the claim. 
 \end{proof}

\paragraph{Identifiability} A concept related to the above two is \emph{identifiability}, introduced by \cite{osband1985providing} and \cite{gneiting2011making}. It requires a property to have a so-called \emph{identification function}:
\begin{definition}[Identification (Id) Function]
A function $V:\R \times \y$ is a \emph{$\p$-identification} (or simply \emph{id}) \emph{function} for property $\Gamma:\p\rightarrow \R$ if for every $P \in \p$:
\[\E_{y \sim P}[V(\gamma,y)] = 0 \Leftrightarrow \Gamma(P) = \gamma.\]
\end{definition}
A $\p$-identification function for $\Gamma$ is said to be \emph{oriented} if it satisfies:
$\E_{y \sim P}[V(\gamma,y)] > 0 \Leftrightarrow \gamma > \Gamma(P)$.
For notational economy we write: $V(\gamma,P) = \E_{\gamma \sim P}[V(\gamma, y)]$.

In other words, the property value can be identified as the zero of its expected id function over the distribution. For oriented identification functions, we further have that over- (resp.\ under-)estimating the property value leads to positive (resp.\ negative) expected identification function values.

\begin{definition}[Identifiable Property]
    A property $\Gamma:\p\rightarrow \R$ is called \emph{identifiable} if there exists a $\p$-identification function for $\Gamma$.
\end{definition}
Intuitively, the connection to elicitability is that if, e.g., the scoring function for a property is convex and differentiable, then its derivative (in the first argument) gives an identification function for the same property.

\paragraph{A connection between elicitability, identifiability and CxLS} Under mild assumptions, elicitability, identifiability, and the CxLS property of continuous distributional properties $\Gamma$ are in fact equivalent, as shown by~\cite{steinwart_elicitation}. Essentially, while Osband's result (Fact~\ref{fact:osband}) states (subject to no assumptions other than $\p$ being convex) that CxLS is a \emph{necessary} condition for elicitability, the characterization of~\cite{steinwart_elicitation} demonstrates that it is also \emph{sufficient}, subject to further technical assumptions (and shows identifiability to be equivalent to elicitability under those same assumptions).


Formally, \cite{steinwart_elicitation} prove their characterization for $\p$ being defined in one of two ways:

\begin{definition}
    \label{def:special_families}
    We let $\p_0$ be defined as the subspace of the Banach space $W_\text{TV}$ which includes all probability distributions whose densities are bounded above almost everywhere. 
    
    We let $\p_{>0}$ be the subspace of $\p_0$ that includes all probability distributions $P \in \p_0$ whose densities are bounded below everywhere by some $\epsilon_P > 0$.
\end{definition}

\begin{theorem}[\cite{steinwart_elicitation}]
\label{thm:characterization}
Consider a space of probability distributions $\p \in \{\p_0, \p_{>0}\}$. $\Gamma:\p\rightarrow \R$ be any continuous, strictly locally non-constant\footnote{\label{footnote:nwlc}`Strictly locally non-constant' is a (weak) requirement that for every $P$ in the interior of $\p$ with $\Gamma(P) = \gamma$, and any $\epsilon$-neighborhood $U_\epsilon$ of $P$ in the metric topology on $\p$, there are distributions $P', P'' \in U_\epsilon$ such that $\Gamma(P'') < \gamma < \Gamma(P')$.}  property. Then the following statements are equivalent:
 \begin{enumerate}
     \item $\Gamma$ is elicitable.
     \item $\Gamma$ has a bounded non-negative order-sensitive scoring function $S$.
     \item $\Gamma$ is identifiable and has a bounded, oriented identification function $V$.
     \item $\Gamma$ has convex level sets: for any $\gamma$ in the range of $\Gamma$, $\{P \in \p : \Gamma(P) = \gamma\}$ is convex.
 \end{enumerate}
 Moreover, if $\Gamma$ is elicitable, then it has a canonical bounded identification function $V^*$ such that every locally Lipschitz continuous order sensitive scoring function $S$ for $\Gamma$ can be written as:
 $$S(\gamma,y) = \int_{\gamma_0}^\gamma V^*(r, y) w(r) dr + \kappa(y)$$
 for some $\gamma_0 \in \range_\Gamma$, some bounded non-negative weighting function $w$ and a function $\kappa$ depending only on the labels. 
\end{theorem}

\subsection{Calibration and multicalibration for property predictors} We now give general definitions of calibration and multi-calibration for predictors of any distributional property in the batch setting. We defer our definitions of sequential multicalibration to Section \ref{sec:sequential}. A variant of these definitions first appeared in \cite{dwork2022beyond} under the name calibration \emph{consistency under mixtures}. This is a generalization of batch mean and quantile calibration error as studied in \cite{hebert2018multicalibration,jung2022batch}.

Fix any dataset over $\z = \x \times \y$, given by its data distribution $D \in \Delta \z$. Suppose that, given any features $x \in \x$, we want to predict the value of property $\Gamma$ on $Y_x$, the label distribution conditional on $x$. For this, we procure a $\Gamma$-\emph{predictor} $f: \x \to \R$. We will call this predictor \emph{$\Gamma$-calibrated} if for all $\gamma \in \range_f$, the conditional label distribution \emph{given} the prediction $f(x) = \gamma$ indeed has property value~$\gamma$. 

\begin{definition}[Calibrated Predictor for Property $\Gamma$]
    A $\Gamma$-predictor $f: \x \to \R$ is
    \emph{$\Gamma$-calibrated} on dataset distribution $D \in \Delta (\x \times \y)$ if for every $\gamma \in \range_f$:
\[\Gamma(Y_{f,\gamma}) = \gamma,\]
where $Y_{f,\gamma} = Y_{\{x : f(x) = \gamma\}}$ is the conditional label distribution induced by $D$ conditional on $f(x) = \gamma$. 
\end{definition}



Now, we extend this definition to that of multicalibration: calibration guarantees that hold with respect to an arbitrary collection of subsets (`groups') of the feature space $\x$.

\begin{definition}[$\g$-Multicalibrated Predictor for Property $\Gamma$] Fix a collection of groups $\g \subseteq 2^\x$.
    A $\Gamma$-predictor $f: \x \to \R$ is
    \emph{$(\g,\Gamma)$-multicalibrated} on dataset distribution $D \in \Delta (\x \times \y)$ if for every $\gamma \in \range_f$ and $G \in \g$:
    \[\Gamma(Y_{f,\gamma,G}) = \gamma,\]
where $Y_{f,\gamma,G} = Y_{\{x : f(x) = \gamma, x \in G\}}$ is the conditional label distribution induced by $D$ given $f(x) = \gamma$ and $x \in G$. 
In other words, a $(\g,\Gamma)$-multicalibrated predictor has the property that the conditional label distribution conditional \emph{both} on the prediction that $f(x) = \gamma$ and on the event that $x$ is a member of group $G$ indeed has property value $\gamma$.
\end{definition}



We will later need to work with a definition of \emph{approximate} multicalibration for predictors with \emph{finite range}. We adopt an $\ell_2$-notion of calibration error, generalizing approximate quantile multicalibration as defined in \cite{jung2022batch}. This guarantee is stronger than the more common $\ell_\infty$-notion of calibration error studied in e.g.~\cite{hebert2018multicalibration,jung_moment,happymap}, and can be related to $\ell_1$ variants of multicalibration error via the Cauchy-Schwarz inequality as discussed in \cite{rothuncertain}.

\begin{definition}[$\alpha$-Approximately $\g$-Multicalibrated Predictor for Property $\Gamma$]
\label{def:multical}
Fix a distribution $D \in \Delta \z$ and a collection of groups $\g \subseteq 2^\x$. For each $G \in \g$, let $\mu(G) = \Pr_{(x,y) \sim D}[x \in G]$ be the probability mass on group $G$. A finite-range predictor $f:\x\rightarrow \range_f$ is \emph{$\alpha$-approximately $(\g,\Gamma)$-multicalibrated} on $D$ if for all $G \in \g$:
\[\sum_{\gamma \in \range_f}\Pr_{(x,y) \sim D}[f(x) = \gamma | x \in G]\left(\gamma - \Gamma(Y_{f,\gamma,G}) \right)^2 \leq \frac{\alpha}{\mu(G)}.\]
Note that $0$-approximate $(\g,\gamma)$-multicalibration is equivalent to our definition of $(\g,\Gamma)$-multicalibration. 
\end{definition}


\section{Sensibility for calibration and elicitability}
\label{sec:sensibility}

We are now ready to make a connection between property elicitation and (multi)calibration.
First we define the notion of a property $\Gamma$ being \emph{sensible} for calibration. 
\begin{definition}[True Distributional Predictor for a Property]
Fix a distributional property $\Gamma: \p \to \R$ and a $\p$-compatible dataset distribution $D$. The \emph{true distributional predictor $f_\Gamma^D$} for $\Gamma$ on $D$ is defined as $f_\Gamma^D(x) = \Gamma(Y_x)$ for $x \in \x$ --- i.e.\ the predictor that for every $x \in \x$ gives the correct value of property $\Gamma$ on the conditional label distribution given $x$. 
\end{definition}

\begin{definition}[Property Sensible for Calibration]
\label{def:sensible}
Fix a property $\Gamma: \p \to \mathbb{R}$, and a collection $\d$ of $\p$-compatible dataset distributions.
We say that $\Gamma$ is \emph{sensible for calibration over $\d$} if the true distributional predictor $f_\Gamma^D$ is $\Gamma$-calibrated on $D$ for all $D \in \d$.
\end{definition}

A key motivation for multicalibration (elaborated on in \cite{dwork2021outcome}) is that we want to produce a predictor $f$ that is indistinguishable from $f_\Gamma^D$ with respect to a class of calibration tests parameterized by $\g$---which only makes sense if $\Gamma$ is sensible for calibration. In general, for properties that are not sensible for calibration, there need not exist calibrated predictors at all (even beyond $f_\Gamma^D$). 

\cite{jung_moment} observed that (in our terminology) \emph{variance} is not sensible for calibration. Here is their example. Suppose $\x = \{x_0, x_1\}$, where $\Pr[X = x_1] = \Pr[X = x_2]=\frac{1}{2}$, and that $Y(x_0)=0$ and $Y(x_1)=1$ with probability one. Then the true distributional predictor has $f_\text{Var}^D(x_0)=f_\text{Var}^D(x_1) = 0$ (as the labels are deterministic). Nevertheless, $\text{Var}(Y | f_\text{Var}^D(x) = 0) = \text{Var}(Y) = 0.25 \neq 0$. In other words, the true label variance is nonzero over the set of points for each of which the label variance \emph{is} 0.
We now significantly generalize and tighten this observation into a characterization that (under mild assumptions) a property is sensible for calibration \emph{if and only if it is elicitable} (or, equivalently, is identifiable/has convex level sets).

We begin by showing that if a property $\Gamma: \p \to \R$ does not have convex level sets on $\p$, then it is not sensible for calibration over any family of $\p$-compatible datasets that includes all possible datasets supported on two points in $\x$ whose respective label distributions belong to $\p$.

\begin{definition}[$2$-Point Dataset Distribution]
    A dataset distribution $D$ over feature-label pairs $(X, Y) \in \x \times \y$ is called \emph{2-point} if there exist two feature vectors $x_1 \neq x_2 \in \x$ such that $\Pr_D[X \not \in \{x_1, x_2\}] = 0$ and $\Pr_D[X = x_1] \neq 0, \Pr_D[X = x_2] \neq 0$ --- in other words, exactly two feature vectors have nonzero probability of occurring under distribution $D$.
\end{definition}

\begin{theorem}[No CxLS $\implies$ Not Sensible for Calibration]
    \label{thm:nocxls_notsensible}
    Consider a property $\Gamma: \p \to \R$ where $\p$ is convex, and any family $\d$ of $\p$-compatible dataset distributions that includes all the $\p$-compatible $2$-point dataset distributions.
    Then, if $\Gamma$ does not have convex level sets on $\p$, it is not sensible for calibration over $\d$.
\end{theorem}
\begin{proof}
    Suppose that $\Gamma$ violates the convex level sets assumption on the distribution family $\p$. Then there exists some value $\gamma \in \range_\Gamma$ such that $\{P \in \p : \Gamma(P) = \gamma\}$ is not convex. Equivalently, there exist distributions $Y_1,Y_2 \in  \p$ such that $\Gamma(Y_1) = \Gamma(Y_2) = \gamma$ but $\Gamma(\lambda Y_1 + (1-\lambda)Y_2) \neq \gamma$ for some $\lambda \in [0,1]$. We now need to exhibit a dataset distribution $D \in \d$ on which the true distributional predictor $f_\Gamma^D$ is not $\Gamma$-calibrated. We construct such a dataset distribution with support over any two feature vectors $x_1 \neq x_2 \in \x$, by setting $Y_{x_1} = Y_1$, $Y_{x_2} = Y_2$, and $\Pr[X = x_1] = \lambda = 1 - \Pr[X = x_2]$. We then immediately see that $\Gamma(Y_{f_\Gamma^D,\gamma}) \neq \gamma$, which completes the proof. 
\end{proof}

To prove the converse, we impose a weak and natural regularity assumption on the dataset distribution $D$: we require that 
the mapping from features $x$ to the corresponding $Y_x$ induced by $D$ be just well-behaved enough that the label distributions $Y_G$ over any subset $G \subseteq \x$  are well-defined as mixtures over the individual label distributions $Y_x$ for $x \in G$. 

In the case of $|\y| < \infty$, the well-behaved nature of the mapping $x \to Y_x$ can be formalized by requiring it to be Lebesgue measurable. In the case when $|\y| = \infty$, the space $W_\text{TV}$ that each $Y_x$ belongs to is a Banach space --- and in this setting, the notions of Lebesgue measurability and integrability (which are only defined in finite-dimensional Euclidean spaces) are replaced by the analogous concepts of Bochner measurability and Bochner integrability (we refer the reader to~\cite{diestel1977vector} for an introduction to these concepts).
Thus, when $|\y| = \infty$, we require the map $x \to Y_x$ to be Bochner measurable.

\begin{definition}[$\p$-Regular Dataset Distribution]
\label{def:regular}
    Fix feature space $\x$, label space $\y$, and a family of probability distributions $\p$ over $\y$. Consider a $\p$-compatible dataset distribution $D$ over $\x \times \y$. Let $\xi_D: \x \to \p$ be defined by $\xi_D(x) = Y_x$, i.e.\ $\xi_D$ is the mapping $x \rightarrow Y_x$ from feature vectors to their label distributions induced by $D$.
    
    The dataset distribution $D$ is called \emph{$\p$-regular} if any of the following is true: 
    \begin{enumerate}
    \item $\x$ is finite;
    \item $\y$ is finite ($|\y| < \infty$) and $\xi_D$ is Lebesgue measurable when $\p$ is viewed as a subset of $\R^{|\y|}$;
    \item $\x, \y$ are infinite and $\xi_D$ is Bochner measurable when $\p$ is viewed as a subset of $W_\text{TV}$.
    \end{enumerate}
\end{definition}

For any $\p$-regular dataset $D$, we now show that a continuous property $\Gamma$ having convex level sets on $\p$ implies that the true distributional predictor $f_\Gamma^D$ is in fact $\Gamma$-calibrated.

\begin{theorem}[CxLS $\implies$ Sensible for Calibration]
    \label{thm:cxls_sensible}
    Consider a continuous property $\Gamma: \p \to \R$, and any family $\d$ of $\p$-regular dataset distributions. Then, if $\Gamma$ has convex level sets on $\p$, it is sensible for calibration over $\d$.
\end{theorem}
\begin{proof}
Suppose that $\Gamma$ has convex level sets on $\p$. We now establish that the true distributional predictor $f^D_\Gamma$ is calibrated on every dataset $D \in \d$: namely, that for all $\gamma \in \range_\Gamma (=\range_{f^D_\Gamma})$, we have $\Gamma(Y_{f^D_\Gamma, \gamma}) = \gamma$. For the remainder of the proof, we fix any $\gamma \in \range_\Gamma$ and any dataset distribution $D \in \d$ (which is $\p$-regular by assumption), and show that $\Gamma(Y_{f^D_\Gamma, \gamma}) = \gamma$.

Let $\ls_\Gamma(\gamma) = \{P \in \p: \Gamma(P) = \gamma\}$ be the $\gamma$-level set of property $\Gamma$, and $Q_\gamma := \{x \in \x: Y_x \in \ls_\Gamma(\gamma)\}$ be the feature space region consisting of all points $x \in \x$ whose label distributions belong to $\ls_\Gamma(\gamma)$. To prove that $\Gamma(Y_{f^D_\Gamma, \gamma}) = \gamma$, we need to establish that $Y_{f^D_\Gamma, \gamma} \in \ls_\Gamma(\gamma)$.

Towards this, we interpret $Y_{f^D_\Gamma, \gamma}$ as a mixture distribution over the individual label distributions $Y_x$ for $x \in Q_\gamma$. Let $\mu_\gamma$ be the probability measure induced by the dataset distribution $D$ on $Q_\gamma$. Then, we can write the formal expression: 
\[Y_{f^D_\Gamma, \gamma} = \E_{x \sim \mu_\gamma}[Y_x] := \int_{Q_\gamma} Y_x d \mu_\gamma.\]


\paragraph{$\x$ is finite:} In this case, we can simply write
$Y_{f^D_\Gamma, \gamma}$ as a convex combination of the constituent distributions $Y_x$ for $x \in Q_\gamma$:
\[Y_{f^D_\Gamma, \gamma} = \sum_{x \in Q_\gamma} \mu_\gamma(x) \cdot Y_x.\]
By the convexity of the level set $\ls_\Gamma(\gamma)$, since $Y_x \in \ls_\Gamma(\gamma)$ for each $x \in Q_\gamma$, then the above convex combination of $Y_x$ for $x \in Q_\gamma$ belongs to $\ls_\Gamma(\gamma)$, thus implying that $Y_{f^D_\Gamma, \gamma} \in \ls_\Gamma(\gamma)$.

\paragraph{$\y$ is finite:} In this case, the expectation $\E_{x \sim \mu}[Y_x]$ is defined as the \emph{Lebesgue integral} of the random variable $Y_x$ over the simplex $\Delta(d) \subset \R^d$. By our assumption that the dataset is $\p$-regular, we have that $Y_x$ is a bounded (e.g. in the $\ell_\infty$ norm) and Lebesgue measurable random variable. Consequently, $Y_x$ is in fact Lebesgue integrable, so the expectation $\E_{x \sim \mu}[Y_x]$ is well-defined and evaluates to some point $u \in \R^d$. It remains to show that $u \in \ls_\Gamma(\gamma)$. 

For this, introduce the indicator function $\1_{\ls_\Gamma(\gamma)}: \Delta(d) \to \{0\} \cup \{+\infty\}$, defined to be $0$ for $Y_x \in \ls_\Gamma(\gamma)$, and $\infty$ otherwise. As the set $\ls_\Gamma(\gamma)$ is convex, its indicator function $\1_{\ls_\Gamma(\gamma)}$ is convex. Therefore, we can apply Jensen's inequality to $\1_{\ls_\Gamma(\gamma)}$ to conclude that:
\[\1_{\ls_\Gamma(\gamma)} (u) = \1_{\ls_\Gamma(\gamma)} \left(\E_{x \sim \mu_\gamma}[Y_x] \right) \leq \E_{x \sim \mu_\gamma} \left[\1_{\ls_\Gamma(\gamma)}(Y_x) \right] = 0,\] implying that $\1_{\ls_\Gamma(\gamma)} (u) = 0$. By definition of $\1_{\ls_\Gamma(\gamma)}$, this demonstrates that $u \in \ls_\Gamma(\gamma)$, as desired.

\paragraph{$\x, \y$ infinite:} In this case, we define $\E_{x \sim \mu_\gamma}[Y_x]$ as the \emph{Bochner integral} (see~\cite{diestel1977vector} for its definition and properties) of the Bochner measurable map $\xi_D$. By a standard Bochner integrability criterion (see Theorem~2 on p.~45 of~\cite{diestel1977vector}), this integral indeed exists and evaluates to a point in the ambient space $W_\text{TV}$, as it is easy to check that $\E_{x \sim \mu_\gamma}[||Y_x||_\text{TV}] < \infty$ (indeed, the TV norm of any probability distribution is $1$ so $\E_{x \sim \mu_\gamma}[||Y_x||_\text{TV}] = 1 < \infty$). Again, we want to show that $Y_{f^D_\Gamma, \gamma} = \E_{x \sim \mu_\gamma} [Y_x] \in \ls_\Gamma(\gamma)$. For this, we use the following result, which can be interpreted as a mean value theorem for Bochner integrals: 
\begin{fact}[Corollary 8 on p.~48 of~\cite{diestel1977vector}]
    \label{fact:bochnerintegral}
    Let $(\Omega, \mathcal{A}, \mu)$ be a finite measure space, $E$ be a Banach space, and $f: \Omega \to E$ be a Bochner $\mu$-integrable function. For $G \subseteq E$, let $\overline{\text{co}}(G)$ denote the closure of the convex hull of $G$. Then, for each $A \in \mathcal{A}$ such that $\mu(A) > 0$, it holds that
    \[\frac{1}{\mu(A)} \int_A f d \mu \in \overline{\text{co}}(f(A)).\]
\end{fact}
To instantiate this fact, we let: (1) $\Omega := \x$, together with the probability measure induced by the dataset over $\x$; (2) the Banach space $E := W_\text{TV}$; (3) the Bochner integrable mapping $f := \xi_D$; and (4) the measurable event $A := Q_\gamma \subseteq \x$.

By definition, $f(A) = f(Q_\gamma) = \text{LS}_\Gamma(\gamma)$. Observe that $\ls_\Gamma(\gamma)$ is convex by assumption, and it is also closed in the standard metric topology on $W_\text{TV}$ since it is the preimage under the continuous mapping $\Gamma$ of the closed singleton $\{\gamma\} \in \range_\Gamma$. Therefore, $f(A)$ is a closed convex set, and thus $\overline{\text{co}}(f(A)) = f(A)$.

Finally, by recalling our definition of $\mu_\gamma$ as the conditional feature vector distribution given $x \in Q_\gamma$ induced by $D$, the integral $\frac{1}{\mu(A)} \int_A f d \mu$ simply becomes $\E_{x \sim \mu_\gamma} [Y_x]$.

Thus, from Fact~\ref{fact:bochnerintegral} we conclude that $Y_{f^D_\Gamma, \gamma} = \E_{x \sim \mu_\gamma} [Y_x] \in  \text{LS}_\Gamma(\gamma)$, as desired.
\end{proof}

Together, Theorems~\ref{thm:nocxls_notsensible} and~\ref{thm:cxls_sensible} establish, under weak regularity conditions, that sensibility for calibration and having convex level sets are equivalent for continuous properties. To finally link sensibility for calibration to elicitability and identifiability, we can now invoke the above discussed Theorem~\ref{thm:characterization} of~\cite{steinwart_elicitation}. Bringing in the extra assumptions on $\p$ (see Definition~\ref{def:special_families}) and $\Gamma$ (the nowhere-locally-constant assumption discussed in Footnote~\ref{footnote:nwlc}) required by Theorem~\ref{thm:characterization}, we obtain our final characterization result:

\begin{restatable}[Sensibility for Calibration ${\scriptstyle\iff}$ Elicitability $\iff$ Identifiability $\iff$ CxLS]{theorem}{sensibilitytheorem}
\label{thm:sensible}
 Let $\Gamma: \p \to \R$ be a continuous and strictly locally non-constant property defined on a convex space of distributions $\p$, where $\p \in \{\p_0, \p_{>0}\}$. Let $\d$ be a family of $\p$-regular dataset distributions over the data domain $\z = \x \times \y$, that includes all the $\p$-compatible $2$-point dataset distributions. 
 
 Then $\Gamma$ is sensible for calibration over $\d$ if and only if $\Gamma$ is elicitable (equivalently, is identifiable, or has convex level sets).
\end{restatable}

\begin{proof}
Under our assumptions on $\Gamma$, $\p$, and $\d$, Theorems~\ref{thm:nocxls_notsensible} and~\ref{thm:cxls_sensible} establish that $\Gamma$ is sensible if and only if it has convex level sets on $\p$. Invoking Theorem~\ref{thm:characterization} of~\cite{steinwart_elicitation} under these assumptions additionally shows that $\Gamma$ has convex level sets on $\p$ if and only if $\Gamma$ is elicitable, and if and only if $\Gamma$ is identifiable. This suffices to show our desired chain of equivalences.
\end{proof}

\section{Batch  multicalibration}
\label{sec:batch}


In this section we give a generic batch $\Gamma$-multicalibration algorithm for elicitable properties $\Gamma$. It is a generalization of (and very similar to) past multicalibration algorithms designed for specific properties, like means~\citep{hebert2018multicalibration,gopalan2022omnipredictors} and quantiles~\citep{jung2022batch}. These algorithms differ in their specifics (the calibration metric they bound, whether they discretize predictions, etc.); we most closely mirror the quantile multicalibration algorithm of~\cite{jung2022batch} that bounds an $\ell_2$ notion of calibration error using a discretized predictor. 

As a reminder, we henceforth focus on bounded properties $\Gamma$ rescaled to have $\range_\Gamma = [0, 1]$. Our algorithm will output \emph{finite-range} multicalibrated $\Gamma$-predictors $f$.
For any integer $m \geq 1$, we denote\footnote{We could have defined $[1/m]$ as any other set of $m$ points in $[0, 1]$ with gaps at most $\frac{1}{m}$ between any two consecutive points.} 
$[1/m] := \{\frac{1}{m+1}, \frac{2}{m+1},\ldots, \frac{m}{m+1}\}$.
Given any $m\geq 1$ (which is effectively our only hyperparameter), our algorithm can produce an $O(\frac{1}{m})$-approximately multicalibrated $\Gamma$-predictor $f$ with $\range_f = [1/m]$. 

By Theorem \ref{thm:characterization}, any continuous elicitable property $\Gamma$ has a strictly consistent scoring function and an identification function. 
We will now need to make a further assumption:
\begin{assumption} \label{assumption:offline}
Assume $\Gamma: \p \to \range_\Gamma$ has an identification function $V$ such that $V(\cdot, Y)$ is strictly increasing and $L$-Lipschitz for each label distribution $Y \in \p$: $|V(\gamma, Y) - V(\gamma', Y)| \leq L |\gamma - \gamma'|$ for all $\gamma, \gamma'$.
\end{assumption}
This assumption is arguably mild. Let $S$ be an antiderivative of $V$, so that $V(\gamma, y) = \frac{\partial S(\gamma, y)}{\partial \gamma}$. Then, $V$ being strictly increasing is equivalent to $S$ being a \emph{convex} strictly consistent scoring function for $\Gamma$. Such a convexity assumption is quite natural in the context of optimization; furthermore, \cite{finocchiaro2018convex} show that  for elicitable properties over finite label spaces $|\y| < \infty$, this is without loss of generality. The extra Lipschitz assumption is what will allow us to quantify our algorithm's convergence rate.

To state Algorithm \ref{alg:batch}, it is convenient for us to re-parameterize our Definition~\ref{def:multical} of $\Gamma$-multicalibration in terms of an id function $V$ for $\Gamma$. (By the properties of $V$, as this updated notion of calibration error goes to $0$, so will the one in Definition~\ref{def:multical}.)
Below, let $Y_{(G, \gamma)}$ be the label distribution conditional on the event $\{x \in \x : f(x) = \gamma, x \in G\}$.

\begin{definition}[Approximate $(\g,V)$-Multicalibration]
\label{def:multicalv}
Fix groups $\g$, a distribution $D$, and an id function $V$ for a property $\Gamma$. A finite-range $\Gamma$-predictor $f: \x \to [0, 1]$ is \emph{$\alpha$-approximately $(\g, V)$-multicalibrated} if for all $G \in \g$: 
\[\sum\limits_{\gamma \in \range_f} \Pr\limits_{x \sim X}[f(x) = \gamma | x \in G] \cdot (V(\gamma, Y_{(G, \gamma)}))^2 \leq \frac{\alpha}{\Pr_{x \in X}[x \in G]}.\]
\end{definition}
Algorithm~\ref{alg:batch} is quite natural. While it can, it finds an intersection $Q_t$ of a group $G \in \g$ and a level set of the current predictor $f$, such that $f$'s prediction on $Q_t$ is too far from the truth, as measured by the magnitude of the expected identification function value over $Q_t$ --- and fixes the situation by shifting $f$'s value on $Q_t$ to the best grid point $\gamma \in [1/m]$.

\begin{algorithm}[H]
\begin{algorithmic}

\STATE \textbf{Initialize $t = 1$ and $f_1 = f$.}
\STATE \textbf{Let} $\alpha = \frac{4L^2}{m}$, \textbf{and let} $V: \range_\Gamma \times \y \to \R$ be an $L$-Lipschitz id function for $\Gamma$ satisfying Assumption \ref{assumption:offline}. 
\WHILE{$f_t$ not $\alpha$-approximately $(\g,V)$-multicalibrated}

  \STATE \textbf{Let $Q_t = \{x: f_t(x) = \gamma, x \in G\}$, where}
  \[(\gamma,G) \in \argmax\limits_{\scriptscriptstyle (\gamma''\!\!,G'\!) \in [\frac{1}{m}] \times \g}  \Pr\limits_{x \in \x} [f_t(x) \!=\! \gamma''\!, x \!\in\! G'] (V(\gamma'', Y_{(\gamma''\!, G')}))^2.\]
  
  \STATE \textbf{Let: $\gamma' = \argmin_{\gamma'' \in [\frac{1}{m}]} \left| V(\gamma'', Y_{Q_t}) \right| $} 
  \STATE \textbf{Update: $f_{t+1}(x) := \1[x \not \in Q_t] \cdot f_t(x) + \1[x \in Q_t] \cdot \gamma'$ for all $x \in \x$, and $t \gets t+1$.}
\ENDWHILE
\STATE \textbf{Output} $f_t$. 

\end{algorithmic}
\caption{BatchMulticalibration($\Gamma, \g, m, f,L$)}
\label{alg:batch}
\end{algorithm}

\begin{restatable}[Guarantees of Algorithm~\ref{alg:batch}]{theorem}{thmbatchalgorithm} \label{thm:batch}
Fix data distribution $D \in \Delta \z$ and groups $\g \subseteq 2^\x$.
Fix a property $\Gamma$ with its scoring function $S$ and id function $V = \frac{\partial S}{\partial \gamma}$ satisfying Assumption~\ref{assumption:offline}, so that $V(\cdot, Y_Q)$ is $L$-Lipschitz on all label distributions $Y_Q$ (for $Q \subseteq \x$) induced by $D$. Set discretization $m \geq 1$. If Algorithm \ref{alg:batch} is initialized with predictor $f_1: \x \to \R$ with score $\E_{(x, y) \sim D} [S(f_1(x), y)] = C_\text{init}$, and $C_\text{opt} = \E_{(x, y) \sim D} [S(f^D_\Gamma(x), y)]$ is the score of the true distributional predictor $f^D_\Gamma$, then Algorithm \ref{alg:batch} produces a $\frac{4 L^2}{m}$-approximately $(\g,V)$-multicalibrated $\Gamma$-predictor $f$ after at most $(C_\text{init} - C_\text{opt}) \frac{m^2}{L}$ updates.
\end{restatable}

The proof of Theorem~\ref{thm:batch} is given in Appendix~\ref{app:batchalgorithm}, and is similar to the analysis of previous algorithms for multicalibration of various properties \citep{hebert2018multicalibration,jung2022batch,happymap}. We use the expected score $\E_D[S]$ (where $V = \frac{\partial S}{\partial \gamma}$) as a potential function for the algorithm: decreases at every step of the algorithm. Prior analyses of mean multicalibration \citep{hebert2018multicalibration} and quantile multicalibration \citep{jung2022batch} used squared loss and pinball loss respectively, as potential functions---these are  strictly proper scoring rules for means and quantiles respectively. Our analysis is the natural generalization of this to arbitrary elicitable properties. The convergence rates follow by showing that $\E[S]$ drops substantially at every iteration.

Note that we have described Algorithm \ref{alg:batch} as able to directly query the expected identification function $V$ on the true data distribution $D$. In practice, we would  instead run it on the empirical distribution over an i.i.d.\ \emph{sample} $\hat{D} \sim D^n$ of $n$ points from $D$.
Appendix~\ref{app:finitesamplebatch} gives finite sample guarantees for this case.

\section{Joint multicalibration}
\label{sec:multi}
Now we take a step towards understanding two interrelated issues: (1) how to multicalibrate vector-valued properties, and (2) how to appropriately extend the notion of multicalibration to some practically important scalar properties that have non-convex level sets and are thus neither elicitable nor sensible for calibration by our Theorem~\ref{thm:nocxls_notsensible} (e.g.\ variance). These are closely related questions, since of course if we have a vector-valued property such that each coordinate is elicitable on its own, then we can simply use our algorithm from Section \ref{sec:batch} to separately multicalibrate each coordinate of the property; vector-valued properties are challenging exactly insofar as their coordinates are not individually elicitable.

Specifically, we study the important case of two-dimensional properties $\Gamma = (\Gamma^0, \Gamma^1)$, where $\Gamma^0$ is elicitable whereas $\Gamma^1$ is \emph{not} elicitable \emph{per se}, but \emph{is} elicitable \emph{conditional} on any fixed value of $\Gamma^0$, and give an algorithm that can produce \emph{jointly multicalibrated} estimators for such pairs. We here list the assumptions we will need on our properties, and define joint multicalibration. 

\begin{definition}[Conditional elicitability]
\label{def:conditional_elicitability}
We say that property $\Gamma^1:\p\rightarrow \R$ is \emph{elicitable conditionally on property $\Gamma^0:\p\rightarrow \R$}, if $\Gamma^1$ is elicitable on each level set of $\Gamma^0$: $\p_{\gamma^0} = \{P \in \p: \Gamma^0(P) = \gamma^0\}$ for all $\gamma^0 \in \range_{\Gamma^0}$.
\end{definition}

For the elicitable component $\Gamma^0$, we denote its scoring and identification functions by $S^0, V^0$.
For property $\Gamma^1$ that is elicitable conditionally on $\Gamma^0$, for each $\gamma^0 \in \range_{\Gamma^0}$ we denote by $V^1_{\gamma^0}: \range_{\Gamma^1} \times \y \to \R$ a function that identifies $\Gamma^1$ on every distribution $P$ such that $\Gamma^0(P) = \gamma^0$, and by $S^1_{\gamma^0}: \range_{\Gamma^1} \times \y \to \R$ a score that is strictly consistent for $\Gamma^1$ on every distribution $P$ such that $\Gamma^0(P) = \gamma^0$.

\paragraph{Assumptions} As in Section~\ref{sec:batch}, we assume that the elicitable component $\Gamma^0$ satisfies Assumption~\ref{assumption:offline}, with $V^0(\gamma^0, \cdot)$ strictly increasing and $L^0$-Lipschitz in $\gamma^0$. Here, we will also need the opposite (similarly mild) assumption:
\begin{assumption}
\label{assumption:antilipschitz}
    $V^0(\cdot, Y)$ is $L^0_a$-anti-Lipschitz around $\Gamma^0(Y)$ for all $Y \in \p$: $|\gamma^0 - \Gamma^0(Y)| \leq L^0_a |V^0(\gamma^0, Y)|$ for all $\gamma^0$.
\end{assumption}

The situation with $\Gamma^1$ is more complex: it has different identification functions $V^1_{\gamma^0}$ for different level sets of $\Gamma^0$, instead of a single function for all $P\in\p$. In general, nothing prevents these functions $V_{\gamma^0}$ from being completely unrelated to each other for different values of $\gamma^0$ (and even undefined on each other's level sets). However, for most properties of interest we can expect that $V^1_{\gamma^0}(\gamma^1, P)$ varies continuously with the parameter $\gamma^0$, and is well-defined even for $P \not \in \{P': \Gamma^0(P') = \gamma^0\}$. To reflect this, and enable our conditional multicalibration algorithm's guarantees, we make the following (mildly stronger) assumption:

\begin{assumption} \label{assumption:conditional_interlevelset}
Assume that $V^1_{\gamma^0}(\gamma^1, P)$ is defined for all $P \in \p$. Assume furthermore that $V^1_{\gamma^0}$ is $L_c$-Lipschitz as a function of $\gamma^0$: that is, for any fixed $\gamma^1$ and $P \in \p$, and for any $\gamma_1^0$ and $\gamma_1^1$, we have $|V^1_{\gamma^0_0}(\gamma^1, P) - V^1_{\gamma^0_1}(\gamma^1, P)| \leq L_c |\gamma^0_0 - \gamma^0_1|$.
\end{assumption}

Further, we assume that the conditional identification functions $V^1_{\gamma^0}$ for $\Gamma^1$ on $\Gamma^0$'s level sets $\{\Gamma^0 = \gamma^0\}$ retain their ``shape'', i.e.\ remain strictly increasing and Lipschitz, even for distributions from other level sets of $\Gamma^0$. While this is a nontrivial assumption to make, in Section~\ref{sec:examples} we verify it when $(\Gamma^0, \Gamma^1)$ is a \emph{Bayes pair}; Bayes pairs are an important and general class of properties \citep{embrechts2021bayes}, and a major use case of our joint multicalibration theory.

\begin{assumption} \label{assumption:conditional_shape}
For all\footnote{In fact, for Algorithm~\ref{alg:conditional} below, which produces predictors discretized over $[1/m] \times [1/m]$, we only need this to hold for $\gamma^0\in [1/m]$, which is less restrictive: $[1/m]$ is a finite set and so does not require Lipschitzness uniformly over all of $\range_{\Gamma^0}$.} $\gamma^0 \in \range_{\Gamma^0}$, assume $V^1_{\gamma^0}(\cdot, P)$ is $L^1$-Lipschitz and strictly increasing for all $P \in \p$.
\end{assumption}

We can now define the central notion of this section: jointly multicalibrated predictors for two-dimensional properties $\Gamma = (\Gamma^0, \Gamma^1)$. Analogously to Definitions~\ref{def:multical}, \ref{def:multicalv} of standard multicalibration, we define this concept in two versions: the first one is parameterized by $(\Gamma^0, \Gamma^1)$, and the second one --- by the identification functions $(V^0, V^1)$. As in Section~\ref{sec:batch}, the latter definition serves to simplify notation in our algorithm analysis (and as this notion of multicalibration error goes to $0$, so does the one parameterized by $(\Gamma^0, \Gamma^1)$).  

In our definition below, we use the following shorthands (for $i = 0, 1$): 
\[\mu_f( \gamma^i | G, \gamma^{1-i}) := \Pr_{x}[f^i(x) = \gamma^i | x \in G, f^{1-i}(x) = \gamma^{1-i}],\] 
and 
\[\mu_f(G, \gamma^i) := \Pr_{x}[x \in G, f^i(x) = \gamma^i].\] 
Also, let $Y_{(G, \gamma^0, \gamma^1)}$ denote the label distribution conditional on the event~\text{$\{x \in \x: f(x) = (\gamma^0, \gamma^1), x \in G\}$.}

\begin{definition}[Approximate Joint Multicalibration]
\label{def:jointmultical}
Fix distribution $D \in \Delta \z$ and group family $\g$. 
Given a property $\Gamma = (\Gamma^0, \Gamma^1)$, a finite-range predictor $f = (f^0, f^1) : \x \to \R^2$ is \emph{$(\alpha^0, \alpha^1)$-approximately $(\g, \Gamma^0,\Gamma^1)$-jointly multicalibrated} if for every $G \in \g$: (1) it holds for all $\gamma^1 \in \range_{f^1}$: 

\[\sum\limits_{\gamma^0 \in \range_{f^0}} \mu_f( \gamma^0 | G, \gamma^1) \cdot (\gamma^0 - \Gamma^0  (Y_{ (G, \gamma^0, \gamma^1)}) )^2 \leq \frac{\alpha^0}{\mu_f(G, \gamma^1)},\]
and (2) for all $\gamma^0 \in \range_{f^0}$, it analogously holds that:
\[\sum\limits_{\gamma^1 \in\range_{f^1}} \mu_f( \gamma^1 | G, \gamma^0) \cdot (\gamma^1- \Gamma^1 (Y_{(G, \gamma^0, \gamma^1)}) )^2 \leq \frac{\alpha^1}{\mu_f(G, \gamma^0)}.\]

Similarly, given identification functions $V^0$, $\{V^1_{\gamma^0}\}_{\gamma^0 \in \range_{\Gamma^0}}$, the predictor $f = (f^0, f^1)$ is \emph{$(\alpha^0, \alpha^1)$-approximately $(\g, V^0,V^1)$-jointly multicalibrated} if for $G \in \g$, $\gamma^1 \in \range_{f^1}$:
\begin{equation} \label{eq:conditional_gamma0}
    \sum\limits_{\gamma^0 \in \range_{f^0}} \mu_f({ \gamma^0 | G, \gamma^1}) \cdot \left(V^0(\gamma^0, Y_{(G, \gamma^0, \gamma^1)}) \right)^2 
    \leq \frac{\alpha^0}{\mu_f(G, \gamma^1)},
\end{equation}
and for all $G \in \g$ and $\gamma^0 \in \range_{f^0}$:
\begin{equation} \label{eq:conditional_gamma1}
    \sum\limits_{\gamma^1 \in \range_{f^1}}  \mu_f( \gamma^1 | G, \gamma^0) \cdot (V^1_{ \Gamma^0  (Y_{(G, \gamma^0, \gamma^1)} )}(\gamma^1, Y_{ (G, \gamma^0, \gamma^1 )}) )^2 \leq \frac{\alpha^1}{\mu_f(G, \gamma^0)}.
\end{equation}
\end{definition}

\paragraph{Summary of the Algorithm} 
We now introduce \texttt{JointMulticalibration} (Algorithm~\ref{alg:conditional}), a canonical algorithm for learning a jointly multicalibrated predictor $f = (f^0, f^1)$ for $(\Gamma^0, \Gamma^1)$. To deal with a two-dimensional property, we employ a two-stage structure whereby we alternately multicalibrate $f^0$ on the \emph{current} level sets of $f^1$, and $f^1$ on the current level sets of $f^0$, until the desired level of joint multicalibration error is reached according to both Equations~\ref{eq:conditional_gamma0} and \ref{eq:conditional_gamma1} in Definition~\ref{def:jointmultical} above.

As in Section~\ref{sec:batch}, our predictors are discretized: $\range_{f^0} = \range_{f^1} = [\frac{1}{m}]$. 
The updates to both $f^0$ and $f^1$ are performed via calls to the subroutine \texttt{BatchMulticalibration$^V$}, which is very similar to \texttt{BatchMulticalibration} (Algorithm~\ref{alg:batch}) except for two differences: (1) to simplify notation, it directly accepts identification functions $V$ rather than properties $\Gamma$; and (2) to satisfy the extra demands of \emph{joint} multicalibration, it has a stricter stopping condition (that is sufficient but not necessary for batch multicalibration as defined in Section~\ref{sec:batch}). We give the  pseudocode for \texttt{BatchMulticalibration$^V$} in Algorithm~\ref{alg:batch_conditional}, and defer its (very similar) analysis to Appendix~\ref{app:conditionalalgorithm} in Lemma~\ref{lem:batch_cond}.

Throughout the execution of Algorithm~\ref{alg:conditional}, the subroutine is invoked on auxiliary group families that consist of pairwise intersections of groups in $\g$ and level sets of either $f^0$ or $f^1$. Since the predictors get updated, the auxiliary groups are always in drift across these invocations, and careful bookkeeping is needed to verify that this does not prevent overall convergence. One key fact we prove towards this is that across \emph{all} invocations on $V^0$ throughout Algorithm~\ref{alg:conditional}, the \texttt{BatchMulticalibration$^V$} subroutine will perform boundedly many updates on $f^0$, implying that also $f^1$ will be re-calibrated at most that many times.

Algorithm~\ref{alg:conditional} significantly generalizes the (mean, moment) multicalibration algorithm of \cite{jung_moment}, leading to some key differences in the analysis. Notably, in our terminology, in their specific case the re-calibration of $f^1$ given $f^0$ can be cast as a single mean multicalibration subroutine using what they call a ``pseudo-label'' technique. At our level of generality, this does not work anymore as we are forced to work with different id functions $V^1_{\gamma^0}$ for $\Gamma^1$ on each level set $\{f^0 = \gamma^0\}$. This is why our inner \texttt{for} loop iterates over the level sets of $f^0$, re-calibrating $f^1$ using $m$ separate invocations of the subroutine (fortunately, these can actually be run in parallel, since $f^0$'s level sets are disjoint). Even with this construction in hand, our potential function argument from Section~\ref{sec:batch} does not easily port over: each level set $\{f^0 = \gamma^0\}$ can overlap with multiple level sets of $\Gamma^0$, so the true property $\Gamma^1$ will generally not admit a single scoring function on $\{f^0 = \gamma^0\}$ that could be used as a potential. This is where our assumptions on the behavior of $V^1_{\gamma^0}$ with respect to $\gamma^0$ crucially enable us to show that, subject to $f^0$ being sufficiently multicalibrated, using the proxy id $V^1_{\gamma^0}$ on the level set $\{f^0 = \gamma^0\}$ will not cause the multicalibration subroutines for $\Gamma^1$ to fail to converge.

We begin by formally defining the subroutine \texttt{BatchMulticalibration$^V$}, and then give the full \texttt{JointMulticalibration} algorithm in Algorithm~\ref{alg:conditional}.

\begin{algorithm}[H]
\begin{algorithmic}

\STATE \textbf{Initialize $t = 1$ and $f_1 = f$.}
\WHILE{$\exists (\gamma, G) \in [1/m] \times \g$ such that $\Pr_{x \in \x} [f_t(x) = \gamma, x \in G] \left(V(\gamma, Y_{(\gamma, G)})\right)^2 \geq \alpha/m$}
  \STATE \textbf{Let $Q_t = \{x: f_t(x) = \gamma, x \in G\}$}
  \STATE \textbf{Let:} 
  \[
    \gamma' = \argmin_{\gamma'' \in [1/m]} \left| V(\gamma'', Y_{Q_t}) \right| 
  \]
  \STATE \textbf{Update: $f_{t+1}(x) := \1[x \not \in Q_t] \cdot f_t(x) + \1[x \in Q_t] \cdot \gamma'$ for all $x \in \x$, and $t \gets t+1$.}
\ENDWHILE
\STATE \textbf{Output} $f_t$. 

\end{algorithmic}
\caption{BatchMulticalibration$^V$($V, \g, m, f, \alpha$)}
\label{alg:batch_conditional}
\end{algorithm}

\begin{algorithm}[H]
\begin{algorithmic}

\STATE \textbf{Let} $V^0$ and $\{V^1_{\gamma^0}\}_{\gamma^0 \in [1/m]}$ be id functions for $\Gamma^0$ and $\Gamma^1$ 
satisfying Assumptions~\ref{assumption:offline}, \ref{assumption:antilipschitz}, \ref{assumption:conditional_interlevelset}, \ref{assumption:conditional_shape}.
\STATE \textbf{Initialize $t = 1$ and $f_1 = (f^0, f^1)$.}

\WHILE{$\exists (\gamma^0, \gamma^1, G) \in [\frac{1}{m}] \times [\frac{1}{m}] \times \g$ s.t.\ $\Pr\limits_{x \in \x} [f_t(x) = (\gamma^0, \gamma^1), x \in G] \left(V^0(\gamma^0, Y_{(\gamma^0, \gamma^1, G)})\right)^2 \geq \frac{\alpha}{m}$}
    \STATE \textbf{Let} $\g^0_t \!\gets\! \{ G \cap \{x \!\in\! \x \!\!: f^1_t(x) \!=\! \gamma^1\} \!: G \!\in\! \g, \gamma^1 \!\!\in\! [\tfrac{1}{m}]\}$ 
    \STATE \textbf{Update} $f^0_{t+1} \gets$ BatchMulticalibration$^V\! (V^0\!, \g^0_t, m, f^0_t\!, \alpha^0)$
    \FOR{$\gamma^0 \in [1/m]$}
        \STATE \textbf{Let} $\g_t^{1, \gamma^0} \!\!\!\gets\! \{ G \cap \{x \in \x: f^0_{t+1}(x) \!=\! \gamma^0\} : G \in \g \}$
        \STATE \textbf{Let} $f^{1, \gamma^0}_{t+1} \!\!\!\gets\!$ BatchMulticalibration$^V\!\! (V^1_{\gamma^0}, \g_t^{1, \gamma^0} \!\!\!, m, f^1_t\!, \alpha^1)$
    \ENDFOR
    \STATE \textbf{Update} $f^1_{t+1}\!(x) \!\gets\!\!\!\!\!\!\!\!\! \sum\limits_{\gamma^0 \in [1/m]} \!\!\!\!\! \1[f^0_{t+1}(x) \!\!=\!\! \gamma^0] \!\cdot\! f^{1, \gamma^0}_{t+1}\!\!(x), \forall x \!\in\!\! \x$
    \STATE \textbf{Update} $t \gets t+1$.
\ENDWHILE
\STATE \textbf{Output} $f_t = (f^0_t, f^1_t)$. 

\end{algorithmic}
\caption{JointMulticalibration($(\Gamma^0, \Gamma^1), \g, m, (f^0, f^1)$)}
\label{alg:conditional}
\end{algorithm}

The following theorem provides the convergence guarantees for Algorithm~\ref{alg:conditional}. The full proof, which rigorously develops the ideas discussed above, is in Appendix~\ref{app:conditionalalgorithm}.

\begin{restatable}[Guarantees of Algorithm~\ref{alg:conditional}]{theorem}{thmconditionalalgorithm}
\label{thm:conditionalalgorithm}
Set $\alpha^0 = \frac{4 (L^0)^2}{m}$ and $\alpha^1 = \frac{4 (L^1)^2}{m}$. Let $\alpha^1_* = \frac{8 ((L^0 L^0_a L_c)^2 + (L^1)^2)}{m}$. Given any $\g \subseteq 2^\x$, $m \geq 1$, \texttt{JointMulticalibration} (Algorithm~\ref{alg:conditional}) outputs an $(\alpha^0, \alpha^1_*)$-approximately $(\g, V^0, V^1)$-jointly multicalibrated predictor $f = (f^0, f^1)$ for the property $(\Gamma^0, \Gamma^1)$, via at most $\frac{B^0 B^1 m^4}{L^0 L^1}$ updates to $f$. Here, $B^0 := \sup_{\gamma, y \in [0, 1]} S^0(\gamma, y) - \inf_{\gamma, y \in [0, 1]} S^0(\gamma, y)$ for $S^0$ an antiderivative of $V^0$, and $B^1 := \max_{\gamma^0 \in [1/m]} \left( \sup_{\gamma, y \in [0, 1]} S^1_{\gamma^0}(\gamma, y) - \inf_{\gamma, y \in [0, 1]} S^1_{\gamma^0}(\gamma, y) \right)$ for each $S^1_{\gamma^0}$ an antiderivative~of~$V^1_{\gamma^0}$.
\end{restatable}

\section{Sequential multicalibration}
\label{sec:sequential}
We now turn to the sequential adversarial setting, in which there is no underlying distribution, and our goal will be to obtain approximate $(\g,V)$-multicalibration (Definition \ref{def:multicalv}) on the \emph{empirical distribution} defined by the transcript $\pi$ of an interaction between the Learner and an Adversary. This generalizes sequential multicalibration for means and quantiles studied by~\cite{gupta2022online} to arbitrary elicitable properties. In fact, even for quantiles, we give a strengthening of the result of~\cite{gupta2022online} --- they give an $\ell_\infty$ variant of calibration that makes use of ``bucketing'' in its conditioning event --- we give a bound on the same $\ell_2$-notion of calibration we use for batch calibration, without any bucketing, which is a strictly stronger guarantee. \cite{onlineomni} similarly obtain this stronger guarantee for sequential mean multicalibration.

\subsection{Setup and preliminaries}

\subsubsection{The sequential learning setting}
In the sequential setting, a \emph{Learner} interacts with an \emph{Adversary} in rounds $t = 1$ to $T$ as follows:
\begin{enumerate}
\item The Adversary chooses a feature vector $x_t \in \x$ and a distribution $Y_t \in \Delta Y$ (possibly subject to some restrictions), and reveals $x_t$ to the Learner.
\item The Learner makes a prediction $p_t \in \R$.
\item The Adversary samples $y_t \sim Y_t$ and reveals $y_t$ to the Learner.
\end{enumerate}
The record of the interaction accumulates in a transcript $\pi = \{(x_t,p_t,y_t)\}_{t=1}^T$. For any $s \leq T$ and transcript $\pi$, the prefix of the transcript $\pi^{<s}$ is defined as  $\pi^{<s} = \{(x_t,p_t,y_t)\}_{t=1}^{s-1}$. We write $\Pi^{<s}$ for the domain of all transcripts of length $< s$. A Learner is a collection of mappings (for each round $t \leq T$)  $\cL_t:\Pi^{<t} \times \x \rightarrow \Delta \R$, and an Adversary is a collection of mappings $\cA_t:\Pi^{<t} \rightarrow \x\times \Delta Y$, specifying their behavior given their observations thus far. 

Now we can introduce our strong, $\ell_2$, definition of online multicalibration that we will then show how to achieve.

\begin{definition}[Online Multicalibration]
Fix a transcript $\pi = \{(x_t,p_t,y_t)\}_{t=1}^{T}$. Let $n(\pi,G) = |\{t : x_t \in G\}|$ denote the number of rounds containing a member of group $G$ in $\pi$, and $n(\pi,\gamma,G) = |\{t : x_t \in G, p_t = \gamma\}|$ denote the number of rounds containing a group $G$ in which the prediction $p_t$ was $\gamma$. 

Fix $\pi$, a collection of groups $\g$, a property $\Gamma$, and an identification function $V$ for $\Gamma$. We say that the transcript $\pi$ is \emph{$\alpha$-approximately $(\g,V)$-multicalibrated} if for all $G \in \g$:
\[\sum_{\gamma}\frac{n(\pi,\gamma,G)}{n(\pi,G)}\left(\sum_{t : p_t = \gamma, x_t \in G}\frac{V(\gamma,y_t)}{n(\pi,\gamma,G)}\right)^2 \leq \alpha \frac{T}{n(\pi,G)}.\]
\end{definition}
\begin{remark}
Observe that this is exactly the definition of approximate multicalibration we gave in Definition \ref{def:multicalv}, in which the empirical distribution over $\pi$ replaces the distribution $\d$. 
\end{remark}
We can simplify the notion of multicalibration somewhat by canceling terms:
\begin{observation}
Fix a transcript $\pi$, a collection of groups $\g$, a property $\Gamma$, and an identification function $V$ for $\Gamma$. For each group $G \in \g$ define the quantity:
\[K_2(G,\pi) = \sum_{\gamma} \frac{1}{n(\pi,\gamma,G)}\left(\sum_{t : p_t = \gamma, x_t \in G}V(\gamma,y_t)\right)^2.\]
Then $\pi$ is $\alpha$-approximately $(\g,V)$-multicalibrated if for all $G \in \g$, $K_2(G,\pi) \leq \alpha T$. 
\end{observation}

In the online setting, our goal will be to control the growth of $K_2(G,\pi)$ as the transcript is generated, for each $G \in \g$. The following Lemma will be key:
\begin{lemma}
\label{lem:onlineincrement}
Fix a partial transcript $\pi^{ < s} = \{(x_t,p_t,y_t)\}_{t=1}^{s-1}$ and a one-round continuation $(x_s,p_s,y_s)$. Write $\pi^{\leq s} = \pi^{<s} \circ (x_s,p_s,y_s)$ for the transcript extended by one round. Define:
\[R(\pi^{<s},G,\gamma) = \sum_{t < s: p_t = \gamma, x_t \in G}V(\gamma,y_t).\]
Then for every $G \in \g$, if $x_s \not\in G$, we have: 
\[K_2(G,\pi^{\leq s}) - K_2(G,\pi^{< s}) = 0.\]
If $x_s \in G$ and $p_s = \gamma$ we have:
\[K_2(G,\pi^{\leq s}) - K_2(G,\pi^{< s})  \leq \frac{1}{n(\pi^{<s},\gamma,G)}\left(2 V(\gamma,y_s) R(\pi^{<s},G,\gamma) + V(\gamma,y_s)^2 \right).\]
\end{lemma}
\begin{proof}
If $x_s \not\in G$, then $K_2(G,\pi^{\leq s}) = K_2(G,\pi^{< s})$ by definition and we are done. Otherwise, if $x_s \in G$ we can calculate:
\begin{eqnarray*}
&& K_2(G,\pi^{\leq s}) - K_2(G,\pi^{< s}) \\
&=&  \frac{1}{n(\pi^{<s},\gamma,G)+1}\left(\left(\sum_{t < s: p_t = \gamma, x_t \in G}V(\gamma,y_t)\right) + V(\gamma,y_s)\right)^2 - \frac{1}{n(\pi^{<s},\gamma,G)}\left(\sum_{t < s: p_t = \gamma, x_t \in G}V(\gamma,y_t)\right)^2 \\
&\leq& \frac{1}{n(\pi^{<s},\gamma,G)}\left(\left(\sum_{t < s: p_t = \gamma, x_t \in G}V(\gamma,y_t)\right) + V(\gamma,y_s)\right)^2 - \frac{1}{n(\pi^{<s},\gamma,G)}\left(\sum_{t < s: p_t = \gamma, x_t \in G}V(\gamma,y_t)\right)^2 \\
&\leq&  \frac{1}{n(\pi^{<s},\gamma,G)}\left(2 V(\gamma,y_s) R(\pi^{<s},G,\gamma) + V(\gamma,y_s)^2 \right).
\end{eqnarray*}
This concludes the proof.
\end{proof}

\subsubsection{A key tool: the Online Minimax Multiobjective Optimization framework}

We will derive our algorithm via the Online Minimax Multiobjective Optimization framework introduced by~\cite{lee2022online}.
\begin{definition}[Online Minimax Multiobjective Optimization Setting]
 A Learner plays against an Adversary over rounds $t \in [T] := \{1, \ldots, T\}$. Over these rounds, the Learner  accumulates a $d$-dimensional loss vector ($d \geq 1$), where each round's loss vector lies in $[-C, C]^d$ for some $C > 0$. At each round $t$, the Learner and the Adversary interact as follows:
\begin{enumerate}
    \item Before round $t$, the Adversary selects and reveals to the Learner an \emph{environment} comprising: 
    \begin{enumerate}
        \item The Learner's and Adversary's respective convex compact action sets $\x^t$, $\y^t$ embedded into a finite-dimensional Euclidean space;
        \item A continuous vector valued loss function $\ell^t(\cdot, \cdot) : \x^t \times \y^t \to [-C, C]^d$, with each $\ell^t_j(\cdot, \cdot): \x^t \times \y^t \to [-C, C]$ (for $j\in [d]$) convex in the 1st and concave in the 2nd argument.
    \end{enumerate}
    \item The Learner selects some $x^t \in \x^t$.
    \item The Adversary observes the Learner's selection $x^t$, and responds with some $y^t \in \y^t$.
    \item The Learner suffers (and observes) the loss vector $\ell^t(x^t, y^t)$.
\end{enumerate}
The Learner's objective is to minimize the value of the maximum dimension of the accumulated loss vector after $T$ rounds---in other words, to minimize: 
$\max_{j \in [d]} \sum_{t \in [T]} \ell^t_j(x^t, y^t).$
\end{definition}

A key quantity in the analysis of the Learner's performance in the online minimax multi-objective optimization setting is the Adversary-Moves-First value of the stage games at each round $t$ of the interaction --- i.e.\ how well the Learner could do if (counter-factually) she knew the Adversary's action ahead of time. 

\begin{definition}[Adversary-Moves-First (AMF) Value at Round $t$]
The \emph{Adversary-Moves-First value} of the game defined by the environment $(\x^t,\y^t,\ell^t)$ at round $t$ is:
\[w^t_A := \sup_{y^t \in \y^t} \min_{x^t \in \x^t} \Big(\max_{j \in [d]} \ell^t_j(x^t, y^t)\Big).\]
\end{definition}

We can measure the performance of the Learner by comparing it to a benchmark defined by the Adversary moves first values of the games defined at each round.

\begin{definition}[Adversary-Moves-First (AMF) Regret]
On transcript $\pi^t \!=\! \{\!(\x^s\!,\y^s\!,\ell^s), x^s\!, y^s \}_{s=1}^t$, we define the Learner's Adversary Moves First (AMF) Regret for the $j^\text{th}$ dimension at time $t$ to be: 
\[R_j^t(\pi^t) := \sum_{s = 1}^t \ell^s_j(x^s, y^s) - \sum_{s=1}^t w^s_A.\]

\noindent The overall \emph{AMF Regret} is then defined as follows: $R^t(\pi^t)= \max_{j \in [d]} R_j^t.$
\end{definition}

\cite{lee2022online} show that in any online minimax multiobjective optimization setting, Algorithm~\ref{alg:general} obtains diminishing AMF regret. 

\begin{algorithm}[H]
\caption{General Algorithm for the Learner that Achieves Sublinear AMF Regret}
\label{alg:general}
\begin{algorithmic}
\FOR{rounds $t=1, \dots, T$}
    \STATE Learn adversarially chosen $\x^t, \y^t$, and loss function $\ell^t(\cdot, \cdot)$.
    \vspace{0.2in}
    \STATE Let 
    \vspace{-0.3in}
    \[\chi^t_j := \frac{\exp\left(\eta \sum_{s=1}^{t-1} \ell_j^s(x^s,y^s)\right)}
    {\sum_{i \in [d]} \exp\left(\eta \sum_{s=1}^{t-1} \ell_i^s(x^s,y^s)\right)} \text{ for $j \in [d]$}.\] 
	\STATE Play 
	\vspace{-0.2in}
	\[x^t \in \argmin_{x \in \x^t} \max_{y \in \y^t} \sum_{j \in [d]} \chi^t_j \cdot \ell^t_j(x, y). \]
	\STATE Observe the Adversary's selection of $y^t \in \y^t$.
\ENDFOR
\end{algorithmic}
\end{algorithm}


\begin{theorem}[AMF Regret guarantee of Algorithm~\ref{alg:general} \citep{lee2022online}] \label{thm:frmwk}
For any $T \geq \ln d$, Algorithm~\ref{alg:general} with learning rate $\eta = \sqrt{\frac{\ln d}{4TC^2}}$ obtains, against any Adversary, AMF regret bounded by: $R^T \leq 4C\sqrt{T \ln d}.$
\end{theorem}

\subsection{The canonical sequential multicalibration algorithm}

In the rest of this section, we show how for any elicitable property $\Gamma$ with a Lipschitz identification function $V$, and for any finite group structure $\g$, the problem of obtaining diminishing $(\g,V)$-multicalibration error in the sequential adversarial  setting can be cast as an instance of online minimax multiobjective optimization, and so can be solved with an appropriate instantiation of Algorithm \ref{alg:general} with multicalibration error bounds following from an appropriate instantiation of Theorem \ref{thm:frmwk}.

We do not need the full power of Assumption \ref{assumption:offline} in this section --- we only need that the identification function $V$ is Lipschitz. We recall that a Lipschitz condition on the identification function depends on the label distribution $Y$, and so the assumption we need will be not only on the property, but on the distribution chosen at each round by the Adversary. The Lipschitz constant can differ from round to round; our assumption will only be on its average value. 

\begin{assumption}
\label{ass:online}
Fix an elicitable property $\Gamma$ with an identification function $V$. Assume that at each round $t$, the Adversary chooses a label distribution $Y_t$ so that $V$ is $L_t$-Lipschitz: 
    \[|V(\gamma, Y_t) - V(\gamma', Y_t)| \leq L_t |\gamma - \gamma'| \quad \text{for all } \gamma, \gamma'.\]
We make no assumption about the individual $L_t$, but assume that their average value is bounded by $L$:
\[\frac{1}{T}\sum_{t=1}^T L_t \leq L.\]
\end{assumption}

We now introduce our canonical algorithm, and prove its guarantees subject to Assumption~\ref{ass:online}.

\begin{algorithm}[H]
\caption{OnlineMulticalibration$(\g,V,m)$}
\label{alg:online}
\begin{algorithmic}
\STATE \textbf{Initialize} an empty transcript $\pi^{\leq 0}.$
\FOR{rounds $t=1, \dots, T$}
    \STATE \textbf{Observe} the Adversary's chosen feature vector $x_t$. 
    \STATE \textbf{Define} the loss function $\ell^t:[1/m]\times \g \rightarrow \mathbb{R}^{m\times |\g|}$ such that for each $G \in \g$ and $\gamma \in [1/m]$:
     \[\ell_{G, \gamma}^t(\gamma_t, y_t) = \mathbbm{1}[x_t \in G,\gamma_t = \gamma] \cdot \frac{1}{n(\pi^{<t},\gamma,G)}\left(2 V(\gamma,y_t) R(\pi^{<t},G,\gamma) + V(\gamma,y_t)^2 \right),\] where:
     \[R(\pi^{<t},G,\gamma) = \sum_{s < t: p_s = \gamma, x_s \in G}V(\gamma,y_s).\]
    \vspace{0.2in}
    \STATE \textbf{Let} 
    \vspace{-0.3in}
    \[\chi^t_{G,\gamma} := \frac{\exp\left(\eta \sum_{s=1}^{t-1} \ell_{G,\gamma}^s(p^s,y^s)\right)}
    {\sum_{(G',m') \in \g\times [1/m]} \exp\left(\eta \sum_{s=1}^{t-1} \ell_{G',m'}^s(p^s,y^s)\right)} \text{ for $(G,\gamma) \in \g \times [1/m]$}.\]
	\STATE \textbf{Let}  
	\vspace{-0.2in}
	\[P^t \in \argmin_{P \in \Delta [1/m]} \max_{y} \sum_{(G,\gamma) \in \g\times [1/m]} \E_{p \sim P^t}
    \left[\chi^t_{G,\gamma} \cdot \ell^t_{G,\gamma}(p, y) \right]. \]
 \STATE \textbf{Sample} $p_t \sim P^t$ \textbf{and make prediction} $p_t$. 
	\STATE \textbf{Observe} the Adversary's selection of $y_t$.
    \STATE \textbf{Update} the transcript $\pi^{\leq t} = \pi^{\leq t-1} \circ (x_t,p_t,y_t)$.
\ENDFOR
\end{algorithmic}
\end{algorithm}

\begin{theorem}[Algorithmic Guarantees for Sequential Multicalibration]
Fix any finite collection of groups $\g$ and any elicitable property $\Gamma$ with a bounded  identification function $V$ satisfying $|V(\gamma,y)| \leq C$. Suppose that the Adversary in the sequential adversarial setting chooses a sequence of distributions that together with $V$ satisfy Assumption \ref{ass:online} with Lipschitz constant $L$. Then for any $m > 0$ there is a randomized algorithm for the Learner (Algorithm \ref{alg:online}) that chooses amongst $m$ discrete predictions at every round and that together with the Adversary induces a transcript distribution that produces a transcript satisfying $\alpha$-approximate $(\g,V)$-multicalibration for:
\[\E_\pi[\alpha] \leq \frac{2CL}{m} +\frac{2C^2\log(T)}{T} + 12C^2\sqrt{\frac{\ln (|\g|m)}{T}}.\]
\end{theorem}
\begin{proof}
We embed our learning problem into the online minimax optimization setting so that we can apply Theorem \ref{thm:frmwk}. First, what are the Learner's and the Adversary's strategy spaces? At each round we let the Learner's strategy space be $\Delta [1/m]$, the simplex of probability distributions over predictions $\gamma_t$ discretized at the granularity of $1/m$. The Adversary's strategy space is the set of all Lipschitz distributions over $\y$. Both of these are convex sets as required. Next, we need to define the loss function $\ell^t$ used at each round. We take the dimension of the loss function to be $d = |\g| m$ --- with a coordinate devoted to each pair $(G \in \g, i \in [m])$. Suppose at round $t$, the Adversary has chosen feature vector $x_t$ (which, recall, is shown to the Learner before she must make a prediction). Then we define the loss vector $\ell^t$ as follows: For each $G \in \g, i \in [m]$ we introduce the loss function
\[\ell_{G, i}^t(Alg_t, Y_t) = \E_{\gamma_t \sim Alg_t, y_t \sim Y_t}
\left[\mathbbm{1} \left[x_t \in G,\gamma_t = \frac{i}{m} \right] \cdot \frac{1}{n(\pi^{<t},\frac{i}{m},G)}\left(2 V(\frac{i}{m},y_t) R(\pi^{<t},G,\frac{i}{m}) + V(\frac{i}{m},y_t)^2 \right) \right],\] 
where $Alg_t \in \Delta [1/m]$ is the distribution over predictions chosen by the Learner, and $Y_t$ is the label distribution chosen by the Adversary. By linearity of expectation, this loss function is linear in the actions of both players, and so in particular is convex-concave as required. By the boundedness of $V$ and the definition of $R$, this loss function takes values in $[-C',C']$ as required, for $C' \leq 3 C^2$.

Next, we need to upper bound the Adversary Moves First value of the game at round $t$:
\[w_t^A = \sup_{Y_t} \min_{Alg_t \in \Delta [\frac{1}{m}]} \max_{G \in \g, i \in [m]}  \E_{\gamma_t \sim Alg_t, y_t \sim Y_t}
\left[\mathbbm{1} \left[x_t \in G,\gamma_t = \frac{i}{m} \right] \cdot \frac{1}{n(\pi^{<t},\frac{i}{m},G)}\left(2 V(\frac{i}{m},y_t) R(\pi^{<t},G,\frac{i}{m}) + V(\frac{i}{m},y_t)^2 \right) \right]. \]

To bound $w_t^A$, consider what the Learner should do if the Adversary first commits to and reveals the true label distribution $Y_t$. The Learner can compute the true property value $\gamma^*_t = \Gamma(Y_t)$. If she could play $\gamma_t = \gamma^*_t$, this would ensure that $V(\gamma_t,Y_t) = 0$, implying that
\[w_t^A = \frac{(V(\gamma^*_t,y_t))^2}{n(\pi^{<t},\gamma^*_t,G)} \leq \frac{C^2}{n(\pi^{<t},\gamma^*_t,G)}.\] 
The Learner cannot generally play $\gamma^*_t$ (since it may not be a multiple of $1/m$ and hence not in her strategy space), but she can select the discrete point $\gamma_t \in [1/m]$ that is closest to $\gamma^*_t$ --- and in particular will satisfy $|\gamma^*_t - \gamma_t| \leq \frac{1}{m}$. With this action, the Learner will achieve $0$ loss in all coordinates corresponding to groups $G$ such that $x_t \not\in G$ as well as in coordinates corresponding to predictions $\frac{i}{m}$ such that $\gamma_t \neq \frac{i}{m}$ (since for each of these coordinates, the indicator $\mathbbm{1}[x_t \in G,\gamma_t = \frac{i}{m}] = 0$). Thus, it remains to consider coordinates corresponding to pairs $(G, i)$ such that $x_t \in G$ and $\gamma_t = \frac{i}{m}$. For any such pair, the indicator $\mathbbm{1}[x_t \in G,\gamma_t = \frac{i}{m}] = 1$, and so the value of the loss in that coordinate can be bounded as:
 \[\E_{y_t \sim Y_t}\left[\frac{1}{n(\pi^{<t},\frac{i}{m},G)}\left(2 V(\frac{i}{m},y_t) R(\pi^{<t},G,\frac{i}{m}) + V(\frac{i}{m},y_t)^2 \right)\right]\leq \frac{2CL_t}{m} + \frac{C^2}{n(\pi^{<t},\frac{i}{m},G)},\]
where we used that by definition, $\E_{y_t \sim Y_t}[V(\gamma^*_t, y_t)] = 0$, that $|\gamma_t - \gamma^*_t| \leq \frac{1}{m}$, and that $V(\cdot,Y_t)$ is $L_t$-Lipschitz by Assumption \ref{ass:online}. 

This upper bounds  the AMF value $w_t^A$, and thus we can apply Theorem \ref{thm:frmwk} to conclude that Algorithm \ref{alg:general} obtains the following AMF regret bound:

\begin{eqnarray*}
&& \max_{G \in \g, i \in [m]}\frac{1}{T}\sum_{t=1}^T \E_{\gamma_t \sim Alg_t, y_t \sim Y_t}
\left[\mathbbm{1} \left[x_t \in G,\gamma_t = \frac{i}{m}\right] 
\cdot \frac{1}{n(\pi^{<t},\frac{i}{m},G)}\left(2 V(\frac{i}{m},y_t) R(\pi^{<t},G,\frac{i}{m}) + V(\frac{i}{m},y_t)^2 \right) \right] \\ 
&\leq& \frac{1}{T}\sum_{t=1}^T \left(\frac{2CL_t}{m} + \frac{\mathbbm{1}[x_t \in G,\gamma_t = \frac{i}{m}]\cdot C^2}{n(\pi^{<t},\frac{i}{m},G)} \right) + 12C^2\sqrt{\frac{\ln (|\g|m)}{T}}\\ 
&\leq& \frac{2CL}{m} +\frac{1}{T}\sum_{t=1}^T \frac{\mathbbm{1}[x_t \in G,\gamma_t = \frac{i}{m}]\cdot C^2}{n(\pi^{<t},\frac{i}{m},G)}  + 12C^2\sqrt{\frac{\ln (|\g|m)}{T}} \\
&\leq& \frac{2CL}{m} +\frac{1}{T}\sum_{t=1}^T \frac{ C^2}{t}  + 12C^2\sqrt{\frac{\ln (|\g|m)}{T}} \\
&\leq& \frac{2CL}{m} +\frac{2C^2\log(T)}{T} + 12C^2\sqrt{\frac{\ln (|\g|m)}{T}}, \\
\end{eqnarray*}
where in the second inequality we use our Lipschitz assumption on the Adversary, and in the second to last inequality we use the fact that on any round in which $x_t \in G$ and $\gamma_t = \frac{i}{m}$, we must have $n(\pi^{\leq t},\frac{i}{m},G) = n(\pi^{<t},\frac{i}{m},G)+1$.

By our choice of loss function and Lemma \ref{lem:onlineincrement}, this implies that for all $G$:
\begin{eqnarray*}
\E[K_2(G,\pi)] &=& \sum_{t=1}^T \E[K_2(G,\pi^{\leq t}) - K_2(G,\pi^{< t})] \leq \frac{2CL}{m} +\frac{2C^2\log(T)}{T} + 12C^2\sqrt{\frac{\ln (|\g|m)}{T}}, 
\end{eqnarray*}
which completes the proof. 
\end{proof}

\section{Applications} \label{sec:examples}
By combining our theory with known results from the elicitation literature in an essentially blackbox way, we now obtain several novel positive and negative results shedding light on an important question: when is it possible to produce multicalibrated predictors for various \emph{risk measures}?  We first summarize our results informally, and then give the formal statements in Section \ref{subsec:formal}.

\paragraph{Joint multicalibration of Bayes pairs and risks} Any elicitable property $\Gamma$ by definition minimizes its scoring function $S$, which, as mentioned, can be interpreted as a \emph{loss} that, when minimized in-expectation over the dataset, yields a predictor for $\Gamma$. For instance, if $\Gamma$ is the \emph{mean}, we would minimize the expected score $\E_{(x, y) \sim D}[S(\gamma, y)]$ for $S(\gamma, y) = (\gamma - y)^2$ --- which is just the familiar \emph{least squares regression}. As another example, a natural score $S_\tau$ for \emph{$\tau$-quantiles} is the well-known \emph{pinball loss} defined as $S_\tau(\gamma, y) := (1-\tau) \gamma + \max\{y - \gamma, 0\}$; its minimization is known as \emph{quantile regression}.

In the context of loss minimization, one may care not only about the minimizer but also about the actual magnitude of the loss, raising the question: how high is the expected loss value at the true property value, i.e.\ $\min_\gamma \E_{(x, y) \sim D}[S(\gamma, y)]$? The answer to this question is captured by the notion of \emph{Bayes risk $\Gamma^B$} of an elicitable property $\Gamma$ with respect to its strictly consistent loss $S$; the two-dimensional property $(\Gamma, \Gamma^B)$ is then known as a \emph{Bayes pair} with respect to the loss $S$.
\begin{definition}[Bayes Risks and Bayes Pairs]
    Fix an elicitable property $\Gamma: \p \to \R$ and a strictly consistent $\Gamma$-scoring function $S$. The \emph{Bayes risk} of $\Gamma$ on $S$ is a property $\Gamma^B: \p \to \R$ given by $\Gamma^B(P) := S(\Gamma(P), P)$ for $P \in \p$. The property $\Gamma^\text{BP} := (\Gamma, \Gamma^B)$ is then called a \emph{Bayes pair with respect to $S$}.
\end{definition}
As an example, (mean, variance) is a Bayes pair with respect to the squared loss $S$; another example, the Bayes pair (quantile, CVaR), will be discussed shortly. Under some natural assumptions, \emph{most} relevant Bayes risks are not elicitable \emph{per se} \citep{embrechts2021bayes}. However, a Bayes risk $\Gamma^B$ is evidently always elicitable \emph{conditionally} on its underlying property $\Gamma$ (as knowing the value of $\Gamma$ fully determines the value of $\Gamma^B$). This makes Bayes pairs a nice use case for our theory of joint multicalibration:
\begin{theorem}[Informal]
\label{thm:bayes_informal}
    Under mild assumptions, all Bayes pairs $\Gamma^\text{BP} := (\Gamma, \Gamma^B)$ with respect to Lipschitz losses $S$ are jointly multicalibratable using Algorithm~\ref{alg:conditional}.
\end{theorem}

\paragraph{CVaR (ES) multicalibration}
Conditional Value at Risk (CVaR), known also as Expected Shortfall (ES), is a tail risk measure of central significance in the financial risk literature. Originally proposed by~\cite{artzner1999coherent}, and introduced into the convex optimization literature by~\cite{rockafellar2000optimization}, it has been at the center of much recent research. Defined for any $\tau \in [0, 1]$ as: 
\[\text{CVaR}_\tau(P) := \E_{Y\sim P}[Y | Y > q_\tau(P)]\] 
(where $q_\tau(P)$ is the $\tau$-quantile of $P$), the Conditional Value at Risk measures the mean of the top $(1-\tau)$-fraction of a random variable's highest values. As such, it provides useful information on tail risk behavior above the corresponding quantile. This, together with a host of other very useful properties --- e.g.\ being a \emph{coherent} risk measure as defined and shown in~\cite{artzner1999coherent} --- makes Conditional Value at Risk a popular and important financial risk measure. The real-world significance of Expected Shortfall is underscored by the fact that in the past decade, it was introduced in international banking regulations, known as the Basel Accords, as a replacement for quantiles (known as Value at Risk (VaR) in finance) for the purposes of market risk capital calculations; see \cite{baselresponse} for details. 

Thus, it is theoretically and practically important to ask whether or not the CVaR is sensible for calibration --- as this would allow us to employ our canonical batch and online multicalibration algorithms to train multicalibrated predictors for the CVaR, thereby complementing the recent algorithmic multicalibration results for quantiles of~\cite{jung2022batch} and~\cite{bastani2022practical}.

The answer to this question turns out to be nuanced. On the negative side, we will show the CVaR to \emph{not} be sensible for calibration, which eliminates the possibility of directly training multicalibrated predictors for it. Fortunately, we demonstrate that this can be remediated by multicalibrating CVaR$_\tau$ not by itself, but rather jointly with the corresponding quantile, $q_\tau$.

\begin{theorem}[Informal]
    For $\tau \in [0, 1]$, $\tau$-CVaR is \emph{not sensible for calibration}. However, $\tau$-CVaR is multicalibratable \emph{jointly with the $\tau$-quantile}, by instantiating Algorithm~\ref{alg:conditional}.
\end{theorem}
For the negative part of this theorem, we simply invoke our Theorem~\ref{thm:nocxls_notsensible} with the result of~\cite{gneiting2011making} that CVaR has nonconvex level sets. For the positive part, we recall a classic result (see e.g.~\cite{frongillo2021elicitation}) that $(q_\tau, \text{CVaR}_\tau)$ is a Bayes pair for $S$ being the (rescaled) $\tau$-pinball loss --- which lets us set up the joint multicalibration algorithm for the pair $(q_\tau, \text{CVaR}_\tau)$ by simply instantiating our above result for general Bayes pairs (Theorem~\ref{thm:bayes_informal}).

\paragraph{An impossibility result for distortion risk measures} Distortion risk measures \citep{distortion1997} are a large, and theoretically and practically important, class of risk measures. A distortion risk measure can be interpreted as first re-weighing a given distribution (via a so-called distortion function) in order to assign more weight to certain outcomes of interest, followed by evaluating the expected value of the modified distribution. Means, quantiles, CVaR, the class of spectral risk measures, and numerous other risk measures of theoretical and practical importance are all instances of distortion risk measures; see e.g.~\cite{kou2016distortion} and~\cite{gzyl2008relationship} for more examples and details.

\begin{definition}[Distortion Risk Measure] 
Given a \emph{distortion function} $h: [0, 1] \to [0, 1]$ (i.e., $h$ is nondecreasing and satisfies $h(0) = 0$ and $h(1) = 1$), the corresponding \emph{distortion risk measure} $\Gamma^h: \p \to \R$ is given by: 
\[\Gamma^h(P) := \int_{-\infty}^0 (h(1-F_P(x)) - 1) dx + \int_0^\infty h(1-F_P(x)) dx\] for $P \in \p$, where $F_P$ is the CDF of $P$. (We assume the integrals exist for all $P \in \p$.) 
\end{definition}
For instance, letting $h(x) = x$ for $x \in [0, 1]$ leads to $\Gamma^h$ being the distribution \emph{mean}; and choosing $h_\tau(x) = \1[x > 1-\tau]$ leads to $\Gamma^{h_\tau}$ being the $\tau$-\emph{quantile}.

As \cite{kou2016distortion} and \cite{wang2015distortion} showed, however, means and quantiles are essentially\footnote{See Definition \ref{def:preciserm} and Theorem \ref{thm:dist_conv} in Section~\ref{subsec:formal} for a precise statement of their result.} \emph{the only} distortion risks with convex level sets on finite-support distributions. By invoking our Theorem~\ref{thm:nocxls_notsensible} (no CxLS $\implies$ not sensible for calibration), we can thus conclude the following sweeping negative result:
\begin{theorem}[Informal]
    No distortion risk measures, other than (essentially) means and quantiles, are sensible for calibration on any dataset family $\d$ which allows for finite-support label distributions.
\end{theorem}
This result tells us that there will not be another multicalibration algorithm for any distortion risk measure: the existing mean and quantile multicalibration algorithms are (essentially) the only ones.

\subsection{Formal statements}
\label{subsec:formal}
\subsubsection{Joint multicalibration of Bayes risks}

Consider any Bayes pair $(\Gamma, \Gamma^B)$ with respect to a strictly consistent scoring function $S(\gamma, y)$. As in Section~\ref{sec:multi}, we assume that $\range_\Gamma \subseteq [0, 1]$ and $\range_{\Gamma^B} \subseteq [0, 1]$. 
To show that Bayes pairs are jointly multicalibratable, we will need to set up several assumptions on the scoring and identification functions associated with $(\Gamma, \Gamma^B)$, in order to ensure the satisfaction of Assumptions~\ref{assumption:offline}, \ref{assumption:antilipschitz}, \ref{assumption:conditional_interlevelset}, and \ref{assumption:conditional_shape} that the generic joint multicalibration result of Section~\ref{sec:multi} relies on.

To satisfy Assumption~\ref{assumption:offline}, we assume that the property $\Gamma$ has an identification function $V$ that is strictly increasing and $L$-Lipschitz in its first argument. To satisfy Assumption~\ref{assumption:antilipschitz}, we additionally assume that $V(\cdot, P)$ is $L_a$-anti-Lipschitz for $P \in \p$. 

Now note that for all $\gamma \in \range_\Gamma$, the Bayes risk $\Gamma^B$ by definition satisfies $\Gamma^B(P) = S(\gamma, y)$ on the level set $\{P \in \p: \Gamma(P) = \gamma\}$ of $\Gamma$. As a result, the identification function for the Bayes risk $\Gamma^B$ on the level set $\{P \in \p: \Gamma(P) = \gamma\}$ can be simply taken to be:
\[V^B_\gamma(\gamma^B, y) := \gamma^B - S(\gamma, y)\] for all $\gamma^B, y \in [0, 1]$.
Taking the expectation over any $P \in \p$, we can thus write the expected conditional identification function of $\Gamma^B$ conditioned on $\Gamma = \gamma$ as $V^B_\gamma(\gamma^B, P) := \gamma^B - S(\gamma, P).$

To satisfy Assumption~\ref{assumption:conditional_interlevelset}, we need to enforce the Lipschitzness of $V^B_\gamma$ be Lipschitz with respect to its subscript $\gamma$. To do so, we assume that the scoring function $S$ for the Bayes pair $(\Gamma, \Gamma^B)$ is $L_S$-Lipschitz in its first argument. For any $\gamma^B$ and any $P$, this lets us write $|V_{\gamma}(\gamma^B, P) - V_{\gamma'}(\gamma^B, P)| = |S(\gamma', P) - S(\gamma, P)| \leq L_S |\gamma -\gamma'|$, implying that $V^B_\gamma$ is $L_S$-Lipschitz in $\gamma$.

Finally, we verify Assumption~\ref{assumption:conditional_shape} of Section~\ref{sec:multi}. Note that the identification function $V^B_\gamma(\cdot, P)$ for the Bayes risk $\Gamma^B$ is well-defined for every $\gamma \in [0, 1]$ and $P \in \p$, even when $\Gamma(P) \neq \gamma$. Furthermore, $V^B_\gamma(\gamma^B, P)$ is \emph{linear} in $\gamma^B$ \emph{with slope $1$}. Thus, $V^B_\gamma(\cdot, P)$ is strictly increasing and, in fact, $1$-Lipschitz for $\gamma \in [0, 1]$ and $P \in \p$, as desired.

With all requisite assumptions on the scoring and identification functions for $\Gamma$ and $\Gamma^B$ satisfied, we can now invoke Theorem~\ref{thm:conditionalalgorithm} to obtain the following joint multicalibration guarantees for Bayes pairs:

\begin{theorem}[Bayes pairs are jointly multicalibratable]
    \label{thm:bayes}
    Consider any Bayes pair $(\Gamma, \Gamma^B)$ with respect to a strictly consistent scoring function $S$. Let $V$ be an identification function for $\Gamma$. Assume that: (1) The scoring function $S$ is $L_S$-Lipschitz in its first argument; (2) $V$ is strictly increasing, $L$-Lipschitz and $L_a$-anti-Lipschitz in its first argument.

    Pick a discretization factor $m \geq 1$. Set $\alpha^0 = \frac{4 L^2}{m}$ and $\alpha^1 = \frac{4}{m}$. Let $\alpha^1_* = \frac{8}{m} ((L L_a L_S)^2 + 1)$. Given any $\g \subseteq 2^\x$, instantiate \texttt{JointMulticalibration} (Algorithm~\ref{alg:conditional}) using the id function $V$ for $\Gamma$, and the id function collection $V^B$ for $\Gamma^B$, such that $V^B_\gamma(\gamma^B, P) := \gamma^B - S(\gamma, P)$ for all $\gamma, \gamma^B \in [0, 1], P \in \p$. 
    
    Then, Algorithm~\ref{alg:conditional} will output an $\left(\frac{4 L^2}{m}, \frac{8}{m} ((L L_a L_S)^2 + 1) \right)$-approximately $(\g, V, V^B)$-jointly multicalibrated predictor $f = (f^0, f^1)$
    for $(\Gamma, \Gamma^B)$, in at most $O \left( \frac{m^4}{L} \right)$ updates\footnote{Specifically, Algorithm~\ref{alg:conditional} will perform at most $R^- R^+ \frac{m^4}{L}$ updates on the predictor $f$, where we have denoted $R^- = \sup\limits_{\gamma, y \in [0, 1]} S(\gamma, y) - \inf\limits_{\gamma, y \in [0, 1]} S(\gamma, y)$ and 
    $R^+ = \frac{1}{2} \max\limits_{\gamma \in [1/m]} \left( \sup\limits_{\gamma^B, y \in [0, 1]} (\gamma^B - S(\gamma, y))^2 - \inf\limits_{\gamma^B, y \in [0, 1]} (\gamma^B - S(\gamma, y))^2 \right).$}
    to the joint predictor $f$.
\end{theorem}

\subsubsection{Joint (quantile, CVaR) multicalibration}

By itself, the CVaR is not sensible for calibration. Using our Theorem~\ref{thm:nocxls_notsensible}, this follows automatically from the classic negative result of~\cite{gneiting2011making}, who shows that CVaR$_\tau$ is not elicitable as it has nonconvex level sets for various distribution families $\p$.

\begin{fact}[CVaR$_\tau$ has nonconvex level sets \citep{gneiting2011making}]
    For any $\tau \in [0, 1]$, CVaR$_\tau$ has nonconvex level sets relative to any class $\p$ of distributions over some interval $I \subseteq \R$ that includes the finite-support distributions, or the finite mixtures of compact-support distributions with well-defined PDF.
\end{fact}

On the positive side, as an easy corollary of Theorem~\ref{thm:bayes}, we obtain our next result that the pair (quantile, CVaR) can be jointly multicalibrated. To be able to apply Theorem~\ref{thm:bayes}, it suffices to identify a strictly consistent scoring function $S_\tau$ for which the pair ($\tau$-quantile, CVaR$_\tau$) for any $\tau \in [0, 1]$ is a Bayes pair, and then obtain the Lipschitz constant for $S_\tau$, as well as the Lipschitz and anti-Lipschitz constants for a strictly increasing identification function $V_\tau$ for the $\tau$-quantile.

And indeed, it is well-known (see e.g.\ Example~1 in~\cite{embrechts2021bayes}) that ($\tau$-quantile, CVaR$_\tau$) is a Bayes pair for a scoring function $S_\tau$ that is the rescaled (by a factor of $\frac{1}{1-\tau}$) \emph{pinball loss}:
\begin{fact}[($\tau$-quantile, CVaR$_\tau$) is a Bayes pair] Fix any $\tau \in [0, 1]$ and let $\Gamma := q_\tau$ be a $\tau$-quantile, and $\Gamma^B := \text{CVaR}_\tau$ be the $\tau$-CVaR. Then $(\Gamma, \Gamma^B)$ is a Bayes pair with respect to the strictly $\Gamma$-consistent scoring function $S_\tau$ defined, for all $\gamma, y \in [0, 1]$, as:
\[S_\tau(\gamma, y) := \gamma + \frac{1}{1-\tau} (y - \gamma)_+,\]
where we have denoted $(u)_+ = \max \{u, 0\}$.
\end{fact}

To bound the Lipschitz constant of $S_\tau$, note that its derivative in the first argument is $\frac{\partial S_\tau(\gamma, y)}{\partial \gamma} = \mathbbm{1}[y \leq \gamma] - \frac{\tau}{1-\tau} \mathbbm{1}[y > \gamma]$. Thus $S_\tau$ has Lipschitz constant $L_{S_\tau} \leq \sup_{\gamma^*, y^*} \left|\frac{\partial S_\tau(\gamma^*, y^*)}{\partial \gamma} \right| = \max \{1, \frac{\tau}{1-\tau}\}$.


Now we need to settle on a strictly increasing (in the first argument) identification function $V_\tau$ for the $\tau$-quantile $q_\tau$ and investigate its Lipschitz properties. Specifically, let us use the standard quantile id function defined as $V_\tau(\gamma, P) := \Pr_{y \sim P}[y \leq \gamma] - \tau$ for all $\gamma$ and all $P \in \p$. Evidently, $V_\tau(\cdot, P)$ is just the CDF of $P$ shifted by $\tau$. Thus, by assuming that all distributions in $\p$ have a strictly increasing CDF, we ensure that $V_\tau$ is strictly increasing in $\gamma$. 

To conveniently quantify the Lipschitzness of $V_\tau$, assume that it is differentiable in $\gamma$: this is equivalent to all $P \in \p$ having a well-defined PDF $pdf_P$, which will then be the derivative of $V_\tau(\cdot, P)$: namely, $\frac{\partial V_\tau(\gamma, P)}{\partial \gamma} = pdf_P(\gamma)$. Therefore, enforcing a Lipschitz and an anti-Lipschitz constant on $V_\tau$ simply translates to assuming an upper and a lower bound on the PDF of the distributions in the underlying family $\p$. Indeed, if we now assume that for all $P \in \p$, the PDF satisfies $0 < M_1 \leq pdf_P(y) \leq M_2 < \infty$ for all $y \in [0, 1]$, this gives us that $V_\tau$ is $M_2$-Lipschitz and $M_1$-anti-Lipschitz.

Plugging the above Lipschitz and anti-Lipschitz bounds on $S_\tau$ and $V_\tau$ into Theorem~\ref{thm:bayes}, we thus obtain the following joint (quantile, CVaR) multicalibration result:

\begin{theorem}[Joint multicalibration of ($\tau$-quantile, CVaR$_\tau$)] Fix any constants $0 < M_1 < M_2$, and take any family $\p$ of probability distributions over $[0, 1]$ such that each $P \in \p$ has a strictly increasing CDF and a well-defined density function $\text{pdf}_P$ satisfying $M_1 \leq \text{pdf}_P(y) \leq M_2$ for all $y \in [0, 1]$.

Fix any target coverage level $\tau \in [0, 1]$, and any group structure $\g \subseteq 2^\x$ on the dataset. Pick a discretization $m \geq 1$.
Set $\alpha^0 = \frac{4 M_2^2}{m}$ and $\alpha^1 = \frac{4}{m}$. Let $\alpha^1_* = \frac{8}{m} ((M_1 M_2 \max\{ 1, \tau/(1-\tau) \})^2 + 1)$.

Then, by appropriately instantiating \texttt{JointMulticalibration} (Algorithm~\ref{alg:conditional}), 
we can compute a \[ \left(\frac{4 M_2^2}{m}, \frac{8}{m} \left( \left(M_1 M_2 \max \left\{ 1, \frac{\tau}{1-\tau} \right\} \right)^2 + 1 \right) \right)-\text{approximately jointly $\g$-multicalibrated predictor}\] 
$f = (f^0, f^1)$ for the pair $(\tau\text{-quantile}, \text{CVaR}_\tau)$, after at most $O \left(\frac{m^4}{M_2} \right)$ updates to the joint predictor $f$.
\end{theorem}

\subsubsection{Sensibility for calibration of distortion risk measures}

We begin by formally stating the result of~\cite{kou2016distortion} and \cite{wang2015distortion} that we will use. It shows that out of all distortion risk measures, the only ones that have convex level sets across the family of all finite-support distributions are: (1) means, (2) quantiles, and (3) two other risk measures which are quantile variants; here are the corresponding definitions.

\begin{definition}
\label{def:preciserm}
Consider any family $\p$ of probability distributions. For any distribution $P \in \p$, let its CDF (which need not be strictly increasing or continuous) be denoted $F_P$. We define the following distributional properties over $\p$:
\begin{enumerate}
    \item For any $\tau \in [0, 1]$, the $\tau$-quantile is defined by: 
    \[q_\tau(P) = \inf\{y: F_P(y) \geq \tau\} \quad \text{ for } P \in \p.\]
    \item For any $\tau \in [0, 1]$ and $c \in [0, 1]$, define the property: 
    \[q^1_{\tau, c}(P) := c \cdot \inf \{y: F_P(y) \geq \tau\} + (1-c) \cdot \inf\{y: F_P(y) > \tau\} \quad \text{ for } P \in \p.\]
    \item For any $\tau \in [0, 1]$ and $c \in [0, 1]$, define the property: 
    \[q^2_{\tau, c}(P) := c \cdot \inf \{y: F_P(y) > 0\} + (1-c) \cdot \inf\{y: F_P(y) = 1\} \quad \text{ for } P \in \p.\]
\end{enumerate}
\end{definition}

Observe that (1) $q^2_{\tau, c}$ is just a convex combination of the $0$\% quantile and the $100$\% quantile of the distribution; and (2) $q^1_{\tau, c}$ in fact \emph{is} (for all $c \in [0, 1]$) the $\tau$-quantile subject to the CDF $F_P$ being \emph{strictly} increasing.

\cite{kou2016distortion} showed that distribution means, together with the three (parametric) properties listed in Definition~\ref{def:preciserm}, are the only distortion risk measures with convex level sets. The proof of this result was later simplified and refined by \cite{wang2015distortion}, who showed that this negative result holds even over the family of distributions with at most 3 points in the support. 

\begin{theorem}[Characterization of distortion risk measures with convex level sets \citep{kou2016distortion, wang2015distortion}]
\label{thm:dist_conv}
    Let $\p_\text{3}$ be the set of all probability distributions supported on at most $3$ real-valued points.
    Let $\p_\text{bd}$ be the set of all bounded distributions over the reals with a well-defined PDF. 
    Let $\p$ be any family of distributions over the reals such that either $\p \supseteq \p_\text{3}$, or $\p \supseteq \p_\text{bd}$.

    Consider any distortion risk measure $\Gamma: \p \to \R$. Then, $\Gamma$ violates the convex level sets assumption on $\p$, unless it is one of the following properties:
    \begin{enumerate}
        \item The distributional mean;
        \item A $\tau$-quantile $q_\tau$, for some $\tau \in [0, 1]$;
        \item The property $q^1_{\tau, c}$, for some $\tau, c \in [0, 1]$;
        \item The property $q^2_{\tau, c}$, for some $\tau, c \in [0, 1]$.
    \end{enumerate}
\end{theorem}

Now, our Theorem~\ref{thm:nocxls_notsensible} lets us immediately conclude that for any $\p$ as in Theorem~\ref{thm:dist_conv}, no distortion risk measure --- other than means, quantiles, or the two parametric properties $q^1_{\tau, c}$ or $q^2_{\tau, c}$ ---  is sensible for calibration over any $\p$-compatible family of dataset distributions $\d$ that includes all the $\p$-compatible $2$-point dataset distributions. To formally restate this:

\begin{theorem}[Sensibility for calibration for distortion risk measures]
    Let $\p_\text{3}$ be the set of all probability distributions supported on at most $3$ real-valued points. Let $\p_\text{bd}$ be the set of all bounded distributions over the reals with a well-defined PDF. 
    Let $\p$ be any convex space of distributions over the reals such that either $\p \supseteq \p_\text{3}$, or $\p \supseteq \p_\text{bd}$.

    Consider any \emph{distortion risk measure} $\Gamma: \p \to \R$, and any family $\d$ of $\p$-compatible dataset distributions that includes all the $\p$-compatible $2$-point dataset distributions.

    Then $\Gamma$ is \emph{not sensible for calibration over $\d$}, unless $\Gamma$ is one of the following properties:
    \begin{enumerate}
        \item The distributional mean;
        \item A $\tau$-quantile $q_\tau$, for some $\tau \in [0, 1]$;
        \item The property $q^1_{\tau, c}$, for some $\tau, c \in [0, 1]$;
        \item The property $q^2_{\tau, c}$, for some $\tau, c \in [0, 1]$.
    \end{enumerate}
\end{theorem}

\subsubsection*{Acknowledgments} We warmly thank Christopher Jung and Arpit Agarwal for enlightening conversations at an early stage of this work. This research was supported in part by the Simons Collaboration on the Theory of Algorithmic Fairness, and NSF grants FAI-2147212 and CCF-2217062.

\bibliography{sample.bib}
\bibliographystyle{plainnat}

\newpage
\appendix

\section{Convergence Guarantees for Batch Multicalibration (Proof of Theorem~\ref{thm:batch})}
\label{app:batchalgorithm}

Our convergence analysis of Algorithm~\ref{alg:batch} will utilize the following natural \emph{potential function}:
\begin{definition}
The \emph{potential} for Algorithm \ref{alg:batch} at round $t$ is:
\[\Phi_t := \E_{(x, y) \sim D} [S(f_t(x), y)] = \E_{x \sim X}[S(f_t(x), Y_x)],\] 
where $f_t: \x \to \R$ is the property predictor at the beginning of iteration $t$ of the algorithm and $S$ is a strictly consistent scoring function for property $\Gamma$ satisfying Assumption \ref{assumption:offline}.  
\end{definition}

First, we prove the following helper Lemma that bounds the change in $S$ --- the potential function of Algorithm~\ref{alg:batch} --- in the scenario where an incorrect prediction $\gamma$ for the property value $\Gamma(Y)$ is corrected on a label distribution $Y$. We will later use this fact to bound the progress of the algorithm after every update to the predictor $f$ for $\Gamma$.
\begin{lemma} \label{lemma:score_bounds}
Consider any property $\Gamma: \p \to \R$. Suppose $S: \range_\Gamma \times \y \to \R$ and $V: \range_\Gamma \times \y \to \R$ with $V(\gamma, y) = \frac{\partial S(\gamma, y)}{\partial \gamma}$ are a strictly consistent scoring function and the corresponding identification function for $\Gamma$ that satisfy Assumption~\ref{assumption:offline}. Then for any $\gamma \in \range_\Gamma$ and any label distribution $Y \in \p$, letting $L_Y$ be the Lipschitz constant of $V(\cdot, Y)$, it holds that:
\[\frac{(V(\gamma, Y))^2}{2 L_Y} \leq S(\gamma, Y) - S(\Gamma(Y), Y) \leq
V(\gamma, Y) (\gamma - \Gamma(Y)) - \frac{(V(\gamma, Y))^2}{2 L_Y}.\]
\end{lemma}
\begin{proof}
We will prove this result with the help of the following claim.
\begin{claim} \label{claim:incrlipshitz}
For any $L$-Lipschitz increasing function $h$ defined on any interval $[a, b]$, it holds that:
\[h(a) (b-a) + \frac{(h(b) - h(a))^2}{2 L} \leq \int_a^b h(t) dt \leq h(b)(b-a) - \frac{(h(b) - h(a))^2}{2 L}.\]
\end{claim}

\begin{proof}
Under these constraints on $h$, the largest value of the integral $\int_{a}^b h(t) dt$ would be obtained if $h(t)$ first increased from $h(a)$ to $h(b)$ for $t \in [a, t']$, where $t' \in [a, b]$ is defined by $(t' - a) L = h(b) - h(a)$, at the fastest rate possible (that is, at the rate $L$), and stayed constant at the value $h(b)$ for $t \in [t', b]$. The integral of this piecewise linear function on $[a, b]$ gives the upper bound.

Conversely, the smallest value of the integral $\int_{a}^b h(t) dt$ would be obtained if $h(t)$ first stayed constant at the value $h(a)$ for $t \in [a, t']$, where $t'$ is defined so that $(b - t') L = h(b) - h(a)$, and then increased from $h(a)$ to $h(b)$ at the fastest rate possible (that is, at the rate $L$) for $t \in [t', b]$. Integrating this function on $[a, b]$ gives the claimed lower bound.
\end{proof}

As V is a derivative of S, by the fundamental theorem of calculus we get: $S(\gamma, Y) - S(\Gamma(Y), Y) = \int_{\Gamma(Y)}^\gamma V(t, Y) dt$.

First assume $\gamma \geq \Gamma(Y)$. By Assumption~\ref{assumption:offline}, $V$ continuously increases from $\Gamma(Y)$ to $\gamma$, and has Lipschitz constant $L_Y$. Then, by Claim~\ref{claim:incrlipshitz}, and using that $V(\Gamma(Y), Y) = 0$, we obtain 
\[\frac{(V(\gamma, Y))^2}{2 L_Y} \leq \int_{\Gamma(Y)}^\gamma V(t, Y) dt \leq V(\gamma, Y) (\gamma - \Gamma(Y)) - \frac{(V(\gamma, Y))^2}{2 L_Y}.\]

Now assume $\gamma < \Gamma(Y)$. Then, we have:
\[ S(\gamma, Y) - S(\Gamma(Y), Y) = \int_{\Gamma(Y)}^\gamma V(t, Y) dt = - \int_\gamma^{\Gamma(Y)} V(t, Y) dt.\]
By Assumption~\ref{assumption:offline}, $V$ continuously increases from $\gamma$ to $\Gamma(Y)$, and has Lipschitz constant $L_Y$.
By Claim \ref{claim:incrlipshitz}, we have:
$ - V(\Gamma(Y), Y) (\Gamma(Y) - \gamma) + \tfrac{(V(\Gamma(Y), Y) - V(\gamma, Y))^2}{2 L_Y}$
$ 
\leq - \int_\gamma^{\Gamma(Y)} V(t, Y) dt 
\leq - V(\gamma, Y) (\Gamma(Y) - \gamma) - \tfrac{(V(\Gamma(Y), Y) - V(\gamma, Y))}{2 L_Y},$
which from $V(\Gamma(Y), Y) = 0$ simplifies to:
$\frac{(V(\gamma, Y))^2}{2 L_Y} 
\leq - \int_\gamma^{\Gamma(Y)} V(t, Y) dt $
$
\leq V(\gamma, Y) (\gamma - \Gamma(Y)) - \frac{(V(\gamma, Y))^2}{2 L_Y}.$ 
Thus, we have shown our bound for both cases $\gamma \geq \Gamma(Y)$ and $\gamma < \Gamma(Y)$.
\end{proof}

Now, we are ready to prove Theorem~\ref{thm:batch}, which gives the convergence rate for Algorithm~\ref{alg:batch}. We restate the theorem here for convenience.

\thmbatchalgorithm*

\begin{proof}
Suppose the algorithm has not halted at round $t$. Thus, $f_t$ does not yet satisfy $\alpha$-approximate $(\g,V)$-multicalibration, so by the pigeonhole principle there is a pair $(G, \gamma) \in \g \times [1/m]$ such that on the set $Q_t := \{x \in \x: x \in g, f_t(x) = \gamma\}$: 
\begin{equation}\label{eq:miscalibration}
|V(\gamma, Y_{Q_t})| \geq \sqrt{\frac{\alpha/m}{\Pr_{x \sim X} [x \in Q_t]}}.
\end{equation}

Now, letting $\gamma' = \argmin_{\gamma'' \in [1/m]} |V(\gamma'', Y_{Q_t})|$, the algorithm will update $f_t \to f_{t+1}$ via the rule:
\[f_{t+1}(x) := \1[x \not \in Q_t] \cdot f_t(x) + \1[x \in Q_t] \cdot \gamma'.\]
From the definition of the potential function values $\Phi_t$ and $\Phi_{t+1}$, we have 
\begin{align*}
    \Phi_{t+1} &= \Pr_{x \sim X}[x \in Q_t] \E_{x \in X}[S(f_{t+1}(x), Y_x) | x \in Q_t] 
    + \Pr_{x \sim X}[x \not\in Q_t] \E_{x \in X}[S(f_{t+1}(x), Y_x) | x \not \in Q_t]
    \\ &= \Pr_{x \sim X}[x \in Q_t] \E_{x \in X}[S(f_{t+1}(x), Y_x) | x \in Q_t] 
    + \Pr_{x \sim X}[x \not\in Q_t] \E_{x \in X}[S(f_t(x), Y_x) | x \not \in Q_t]
    \\ &= \Phi_t + \Pr_{x \sim X}[x \in Q_t] \left( \E_{x \in X}[S(f_{t+1}(x), Y_x) - S(f_t(x), Y_x) | x \in Q_t] \right)
    \\ &= \Phi_t + \Pr_{x \sim X}[x \in Q_t] \left( \E_{x \in X}[S(\gamma', Y_x) - S(\gamma, Y_x) | x \in Q_t] \right)
    \\ &= \Phi_t + \Pr_{x \sim X}[x \in Q_t] \left( S(\gamma', Y_{Q_t}) - S(\gamma, Y_{Q_t}) \right).
\end{align*}
Here Step 2 follows because $f_t(x) = f_{t+1}(x)$ for all $x$ outside $Q_t$, and Step 5 uses the fact that $f_t$ and $f_{t+1}$ are both constant on $Q_t$ to rewrite expected scores of $f_t, f_{t+1}$ over $x \sim X$ simply as the scores of the predictor values $\gamma, \gamma'$ with respect to the mixture distribution $Y_{Q_t}$ of labels over the region $Q_t$.

From here, we have:
\begin{align*}
    \Phi_{t+1} - \Phi_t &= \Pr_{x \sim X}[x \in Q_t] \Big( (S(\gamma', Y_{Q_t}) - S(\Gamma(Y_{Q_t}), Y_{Q_t})) - (S(\gamma, Y_{Q_t}) - S(\Gamma(Y_{Q_t}), Y_{Q_t})) \Big)
    \\ & \leq \Pr_{x \sim X}[x \in Q_t] \Big( V(\gamma', Y_{Q_t}) (\gamma' - \Gamma(Y_{Q_t})) - \frac{(V(\gamma', Y_{Q_t}))^2}{2 L}  - \frac{(V(\gamma, Y_{Q_t}))^2}{2 L} \Big)
    \\ & \leq \Pr_{x \sim X}[x \in Q_t] \Big( V(\gamma', Y_{Q_t}) (\gamma' - \Gamma(Y_{Q_t})) - \frac{(V(\gamma', Y_{Q_t}))^2}{2 L} \Big)  - \frac{\alpha}{2 L m}
    \\ & \leq \Pr_{x \sim X}[x \in Q_t] \Big( V(\gamma', Y_{Q_t}) (\gamma' - \Gamma(Y_{Q_t})) \Big)  - \frac{\alpha}{2 L m}
    \\ & \leq V(\gamma', Y_{Q_t}) (\gamma' - \Gamma(Y_{Q_t})) - \frac{\alpha}{2 L m}
    \\ & \leq L | \gamma' - \Gamma(Y_{Q_t}) | \cdot |\gamma' - \Gamma(Y_{Q_t})| - \frac{\alpha}{2 L m}
    \\ & \leq \frac{L}{m^2} - \frac{\alpha}{2 L m}.
\end{align*}
The equality follows by introducing an added and subtracted term $S(\Gamma(Y_{Q_t}), Y_{Q_t})$.
The 1st inequality applies the upper and lower bound of Lemma~\ref{lemma:score_bounds} to the two score differences.
The 2nd inequality follows by the $\alpha$-miscalibration condition of Equation~\ref{eq:miscalibration}. The 3rd inequality drops the nonpositive term $- \Pr_{x \sim X}[x \in Q_t]\frac{(V(\gamma', Y_{Q_t}))^2}{2 L}$. The 4th inequality drops the factor $\Pr_{x \sim X}[x \in Q_t] \leq 1$. The 5th inequality is because by the Lipschitzness of $V$, we have $|V(\gamma', Y_{Q_t})| = |V(\gamma', Y_{Q_t}) - V(\Gamma(Y_{Q_t}), Y_{Q_t})| \leq L |\gamma' - \Gamma(Y_{Q_t})|$. The 6th inequality is because $\gamma'$ must be at most $\frac{1}{m}$ away from the true property value $\Gamma(Y_{Q_t})$ on $Q_t$: farther grid points $\gamma''$ would result in a worse $|V(\gamma'', Y_{Q_t})|$ (and would thus contradict the choice of $\gamma'$ by the algorithm), by the structure of the $m$-discretization and the monotonic increase of $|V(\cdot, Y_{Q_t})|$ in both directions away from $\Gamma(Y_{Q_t})$.

Now setting $\alpha = \frac{4 L^2}{m}$, we get 
$\Phi_{t+1} - \Phi_t \leq \frac{L}{m^2} - \frac{\alpha}{2 L m} = -\frac{L}{m^2} = -\frac{L}{m^2}$,
and telescoping this over the rounds $t = 1, \ldots, T$, where $T$ is the total number of iterations before convergence, we obtain:
\[\Phi_T - \Phi_1 \leq -T \frac{L}{m^2}.\]
By assumption $\Phi_1 = C_\text{init}$ and $\Phi_T \geq C_\text{opt}$, so we have $T \frac{L}{m^2} \leq \Phi_1 - \Phi_T \leq C_\text{init} - C_\text{opt}$, and thus \[T \leq (C_\text{init} - C_\text{opt}) \frac{m^2}{L},\]
concluding the proof.
\end{proof}


\section{Finite Sample Guarantees for Batch Multicalibration}
\label{app:finitesamplebatch}

We have described Algorithm \ref{alg:batch} as if it has direct access to the underlying distribution $D$ (since it computes expectations of the identification function $V$ on the underlying distribution). In general we do \emph{not} have access to $D$ directly, and instead have access only to a sample $\hat{D} \sim D^n$ of $n$ points sampled i.i.d. from $D$. In practice, we would run the algorithm on the empirical distribution over the $n$ points in $\hat{D}$, and its guarantees would carry over to the underlying distribution $D$ from which the points were sampled. \cite{jung2022batch} proved this for the special case of quantiles (in which the identification function $V$ is the \emph{pinball loss}), but in fact their proof uses nothing other than the conditions in Assumption \ref{assumption:offline}. We state the more general version of the theorem here (implicit in \cite{jung2022batch}) and briefly sketch the argument. We note that this argument is to establish that Algorithm \ref{alg:batch} generalizes when used as an empirical risk minimization algorithm. An alternative means to obtain generalization bounds would be to follow the strategy of \cite{hebert2018multicalibration} and use techniques from adaptive data analysis to implement a statistical query oracle, and to then modify the algorithm so as to compute the quantities $V(\gamma,Q_t)$ only through this oracle---this would also work to give similar bounds. 

\begin{theorem}[Implicit in \cite{jung2022batch}]
\label{thm:gen1}
Fix a distribution $D \in \Delta \z$ and a property $\Gamma$ together with a bounded  identification function $V$. Suppose Algorithm \ref{alg:batch} is run using the empirical distribution on a dataset $\hat{D} \sim D^n$ consisting of $n$ i.i.d. samples from $D$. Then if Algorithm \ref{alg:batch} halts after $T$ rounds and returns a model $f_T$, with probability $1-\delta$ over the randomness of the data distribution, $f^T$ satisfies $\alpha$-approximate $(\g,V)$-multicalibration with respect to $D$ for:
\[\alpha = \frac{4L^2}{m}+O\left( \sqrt{\frac{\ln(1/\delta)+T\ln(m^2|\g|)m}{n}}+\frac{m\ln(1/\delta)+T\ln(m^2|\g|)}{n} \right)\]
\end{theorem}
The proof has a simple structure. Since $V$ is bounded, expectations of $V$ over $D$ (i.e.\ the quantities of the form $V(\gamma,Y_{G,\gamma})$ that appear in the definition of approximate $(\g,V)$-multicalibration) concentrate around their expectations with high probability when evaluated on the empirical distribution $\hat{D}$. Thus for any \emph{fixed} model $f_t$, we can establish that its $(\g,V)$-multi-calibration error is similar in and out of sample by union bounding over each $G \in \g$ and $\gamma \in \range_{f_t} = [1/m]$. Of course the model $f_T$ output by the algorithm is not fixed before $\hat{D}$ is sampled, so to establish the claim, it is necessary to union bound over \emph{all} models $f_T$ that might be output. But we can do this; because Algorithm \ref{alg:batch} produces models with range restricted to $[1/m]$, fixing a model $f_t$, we can count how many models $f_{t+1}$ might result at the next step --- at most one for every choice of $G \in \g$, $\gamma \in [1/m]$, and $\gamma' \in [1/m]$, so at most $|G|m^2$ many models. Thus fixing some initial model $f_1 = f$, if the algorithm halts after $T$ steps, the number of models $f_T$ that might be output is bounded by $(|G|m^2)^T$. The theorem then follows by union bounding over all such models. 

Theorem \ref{thm:gen1} upper bounds the generalization error of Algorithm \ref{alg:batch} in terms of the number of rounds $T$ before it halts. Thus, paired with an upper bound on $T$ it gives a worst-case bound on generalization error. Theorem \ref{thm:batch} upper bounds the round complexity by $T \leq O\left(\frac{m^2}{L} \right)$, but there is a catch: $L$ here is the Lipschitz constant for expectations of $V$ taken over the true underlying distribution $D$, and this will generally not be preserved over the empirical distribution $\hat{D}$. Nevertheless, \cite{jung2022batch} show that the same convergence bound holds when run on $\hat{D}$ (up to constants) --- by arguing that each round of the algorithm run on $\hat{D}$ decreases the potential function as measured on $D$ (where the Lipschitz assumption has been made). This is because the algorithm decides on its update each round by measuring quantities of the form $V(\gamma,Q)$ which are expectations of a bounded function $V$, and so concentrate around their true values. 

\begin{theorem}[Implicit in \cite{jung2022batch}]
Fix a distribution $D \in \Delta \z$ (which induces a set of conditional label distributions $Y_Q$ for each $Q \subset \x$) and a property $\Gamma$ together with an identification function $V$. Assume that $\Gamma$ and $V$ together with the set of label distributions $\p = \{Y_Q : Q \subset \x\}$ together satisfy Assumption~\ref{assumption:offline} with Lipschitz constant $L$.  Suppose Algorithm \ref{alg:batch} is run using the empirical distribution on a dataset $\hat{D} \sim D^n$ consisting of $n$ i.i.d. samples from $D$. Then for any $\delta > 0$, if:
\[n \geq \Omega \left(\ln\left(\frac{m^2}{L\delta}\right) + \frac{m^2}{L}\ln\left(\frac{|G|m}{L} \right) \frac{m^4}{L^2}\right) \]
with probability $1-\delta$ over the randomness of $D$, Algorithm \ref{alg:batch} halts after at most $T = O\left(\frac{m^2}{L} \right)$ many steps. 
\end{theorem}
Together with Theorem \ref{thm:gen1}, this establishes a worst-case generalization bound for batch property multicalibration that is polynomial in all of the parameters of the problem and the assumed Lipschitz constant $L$ of the property's identification function $V$.


\section{Joint Multicalibration Guarantees}
\label{app:conditionalalgorithm}

We here state the guarantees enjoyed by Algorithm~\ref{alg:batch_conditional}. The statement is stronger than that of the guarantees for the similar Algorithm~\ref{alg:batch}, in two ways: 1) The notion of achieved multicalibration error at convergence (Equation~\ref{eq:batch_cond}) is stronger than that of Algorithm~\ref{alg:batch}. 2) We show that even with the input group family $\g$ not fixed beforehand (and thus potentially changing over time), Algorithm~\ref{alg:batch_conditional} will never perform more than a certain number of updates to the predictor $f$, and if it does perform that many updates then it will be approximately calibrated conditional on \emph{all} measurable subsets of $\x$ (rather than just the ones it explicitly performed updates on).

\begin{restatable}{lemma}{lembatchcalibratedeverywhere} \label{lem:batch_cond}
Set $\alpha = \frac{4L^2}{m}$. \texttt{BatchMulticalibration$^V$} (Algorithm~\ref{alg:batch_conditional}), when run on a function $V$ that is monotonically increasing and $L$-Lipschitz in its first argument, outputs a $[1/m]$-discretized predictor $f$ that satisfies: 
\begin{equation} \label{eq:batch_cond}
    \Pr\limits_{x \in \x} [f(x) = \gamma, x \in G] \left(V(\gamma, Y_{(\gamma, G)})\right)^2 \leq \frac{\alpha}{m} \quad \text{for all $\gamma \in [1/m], G \in \g$}.
\end{equation}
Moreover, Algorithm~\ref{alg:batch_conditional} terminates in at most $\frac{Bm^2}{L}$ iterations, where $B = \sup_{\gamma, y \in [0, 1]} S(\gamma, y) - \inf_{\gamma, y \in [0, 1]} S(\gamma, y)$ for $S$ an antiderivative of $V$, and if it runs for that long, the resulting predictor will satisfy (\ref{eq:batch_cond}) for \emph{all} (measurable) regions $G \subseteq \x$.
\end{restatable}

\begin{proof}
Denote by $S$ an antiderivative of $V$, and define, similar to the proof of Theorem~\ref{thm:batch}, the potential value at iteration $t$ of Algorithm~\ref{alg:batch_conditional} as:
\[\Phi_t := \E_{(x, y) \sim D} [S(f_t(x), y)] = \E_{x \sim X}[S(f_t(x), Y_x)].\] Suppose that at round $t$, the algorithm finds a violation of its \texttt{while} loop condition for some $G \in \g$ and $\gamma \in [1/m]$. Let $Q_t = \{x \in \x: x \in \g, f(x) = \gamma\}$. Via the same calculations as in the proof of Theorem~\ref{thm:batch}, we have that 
\begin{align*}
    \Phi_{t+1} - \Phi_t &\leq \Pr_{x \sim X}[x \in Q_t] \Big( V(\gamma', Y_{Q_t}) (\gamma' - \gamma^*_t) - \frac{(V(\gamma, Y_{Q_t}))^2}{2 L} \Big)
    \\ &\leq \Pr_{x \sim X}[x \in Q_t] \Big( |V(\gamma', Y_{Q_t})| |\gamma' - \gamma^*_t| - \frac{(V(\gamma, Y_{Q_t}))^2}{2 L} \Big)
    \\ &\leq \Pr_{x \sim X}[x \in Q_t] \Big( L |\gamma' - \gamma^*_t|^2 - \frac{(V(\gamma, Y_{Q_t}))^2}{2 L} \Big),
\end{align*}
where we denote by $\gamma^*_t$ the unique point such that $V(\gamma^*_t, Y_{Q_t}) = 0$ (it is the analog of $\Gamma(Y_{Q_t})$ in the proof of Theorem~\ref{thm:batch}).
This argument is still valid as it rests on Lemma~\ref{lemma:score_bounds}, which requires properties of $V$ and $S$ that are still satisfied here.

Now, just as in the aforementioned proof, we have by the monotonicity of $V(\cdot, Y_{Q_t})$ that since the algorithm chooses $\gamma' = \argmin_{\gamma'' \in [1/m]} \left| V(\gamma'', Y_{Q_t}) \right|$,  it must be that $|\gamma' - \gamma^*_t| \leq \frac{1}{m}$, and so we obtain
\[\Phi_{t+1} - \Phi_t \leq \frac{L}{m^2} - \frac{1}{2L} \Pr_{x \sim X}[x \in Q_t] |V(\gamma, Y_{Q_t})|^2. \] 
Since the condition of the \texttt{while} loop demands that $\Pr_{x \sim X}[x \in Q_t] \left|V(\gamma, Y_{Q_t})\right|^2 \geq \alpha/m$, we get 
\[\Phi_{t+1} - \Phi_t \leq \frac{L}{m^2} - \frac{\alpha}{2Lm}. \] 
Setting $\alpha = \frac{4L^2}{m}$, we then have $\Phi_{t+1} - \Phi_t \leq -\frac{L}{m^2}$, and thus, by telescoping, $\Phi_T - \Phi_0 \leq -\frac{T L }{m^2}$. Since by definition $B = \sup_{\gamma, y \in [0, 1]} S(\gamma, y) - \inf_{\gamma, y \in [0, 1]} S(\gamma, y)$, we also have $\Phi_T - \Phi_0 \geq -B$, and therefore $T \leq \frac{B m^2}{L}$, providing an upper bound on the number of iterations of the algorithm.

Importantly, in this argument we never referenced the actual definition of $Q_t$ (i.e.\ that $Q_t = \{x \in \x: x \in \g, f(x) = \gamma\}$) --- we only used that it satisfies the \texttt{while} loop condition, i.e.\ $\Pr_{x \sim X}[x \in Q_t] \left|V(\gamma, Y_{Q_t})\right|^2 \geq \alpha/m$. Therefore, the upper bound $\frac{B m^2}{L}$ on the total number of iterations in fact holds for any 
arbitrary sequence of regions $Q_1, Q_2, \ldots$ where each $Q_i \subseteq \x$ is measurable with respect to the marginal data distribution over $\x$. As a result, we know that \emph{if} \texttt{BatchMulticalibration$^V$} does run for at least $\frac{B m^2}{L}$ iterations, then as soon as it finishes iteration $t = \frac{B m^2}{L}$, there will not exist \emph{any} measurable $Q \subseteq \x$ violating condition~(\ref{eq:batch_cond}). Thus, no matter which group family $\g$ \texttt{BatchMulticalibration$^V$} is run on (and even if the group family were to change arbitrarily during the execution), it will never update the predictor $f$ more than $\frac{B m^2}{L}$ times, concluding the proof.
\end{proof}

\thmconditionalalgorithm*

\begin{proof}
\textbf{Runtime:} First, observe that the \texttt{while} loop in Algorithm~\ref{alg:conditional} will stop after at most $\frac{B^0 m^2}{L^0}$ iterations if we set $\alpha^0 = \frac{4(L^0)^2}{m}$. Indeed, all invocations of \texttt{BatchMulticalibration$^V$} on $f^0$ with the identification function $V^0$ can be pieced together into a single process that first multicalibrates $f^0$ with respect to $\g^0_1$, then takes the resulting predictor and multicalibrates it with respect to $\g^0_2$, and so on until the stopping condition of the \texttt{while} loop in \texttt{JointMulticalibration} is met. This is equivalent to a single run of \texttt{BatchMulticalibration$^V$} where the group family is externally updated from time to time: $\g^0_1 \to \g^0_2 \to \ldots \to \g^0_t \to \ldots$. But by Lemma~\ref{lem:batch_cond}, this process cannot perform a total of more than $\frac{B^0 m^2}{L^0}$ updates on the predictor $f^0$. Since the predictor $f^0$ is updated at least once in each iteration of the \texttt{while} loop of \texttt{JointMulticalibration}, this also bounds the number of iterations of the \texttt{while} loop. 

Now, for each iteration of the \texttt{while} loop, we have $m$ calls to \texttt{BatchMulticalibration$^V$} as applied to all identification functions $V^1_{\gamma^0}$ for $\gamma^0 \in [1/m]$. Again by Lemma~\ref{lem:batch_cond}, each of them takes at most $\frac{B^1 m^2}{L^1}$ updates to converge. This follows directly from Assumption~\ref{assumption:conditional_shape}, which states that $V_{\gamma^0}(\cdot, P)$ is $L^1$-Lipschitz and monotonically increasing for all $P \in \p$ (not just for $P$ such that $\Gamma^0(P) = \gamma^0$). 
Naively, running the subroutine $m$ times, once for each level set of $f^0_{t+1}$, would amount to a total of $m \cdot \frac{B^1 m^2}{L^1} = \frac{B^1 m^3}{L^1}$ iterations. But in fact, all these $m$ invocations can be viewed as a single invocation of \texttt{BatchMulticalibration$^V$} that updates the predictor $f^1$ for $\Gamma^1$ using an identification function $V^1_*$ defined as $V^1_{\gamma^0}$ on each level set $\{f^0_{t+1} = \gamma^0\}$ (which is well defined since these level sets partition the domain $\x$). Therefore, there will be only at most $\frac{B^1 m^2}{L^1}$ across \emph{all} these $m$ invocations of \texttt{BatchMulticalibration$^V$}. 

Taking the above observations together, Algorithm~\ref{alg:conditional} will therefore terminate after at most $\frac{B^0 m^2}{L^0} \cdot \frac{B^1 m^2}{L^1} = \frac{B^0 B^1 m^4}{L^0 L^1}$ updates to $f^0$ and to $f^1$, as claimed.

\paragraph{Multicalibration Guarantees:} 
Now, we show that the predictors $f^0_T, f^1_T$ output at termination satisfy the conditions of Definition~\ref{def:jointmultical} of approximate joint multicalibration.

By the stopping condition of the \text{while} loop, at termination we have for all $\gamma^0, \gamma^1, G$ that $\Pr\limits_{x \in \x} [f_T(x) = (\gamma^0, \gamma^1), x \in G] \left(V^0(\gamma^0, Y_{(\gamma^0, \gamma^1, G)})\right)^2 \leq \frac{\alpha^0}{m}$, implying after dividing by $\Pr_{x \in \x}[x \in G, f^1_T(x) = \gamma^1]$ that
\begin{equation} \label{eq:gamma0root}
    \Pr\limits_{x \in \x} [f^0_T(x) = \gamma^0 | x \in G, f^1_T(x) = \gamma^1] \left(V^0(\gamma^0, Y_{(\gamma^0, \gamma^1, G)})\right)^2 \leq \frac{\alpha^0/m}{\Pr[x \in G, f^1_T(x) = \gamma^1]} \text{ for all } G \in \g, \gamma^1 \in \range_{f^1_T}.
\end{equation}
For every $G \in \g$ and $\gamma^1 \in \range_{f^1_T}$, summing this inequality over all at most $m$ values $\gamma^0 \in \range_{f^0_T}$, we obtain that:
\[\sum_{\gamma^0 \in \range_{f^0_T}} \Pr_{x \sim X}[f^0_T(x) = \gamma^0| x \in G, f^1_T(x) = \gamma^1] \cdot \left(V^0(\gamma^0, Y_{(G, \gamma^0, \gamma^1)}) \right)^2 \leq \frac{\alpha^0}{\Pr_{x \in X}[x \in G, f^1_T(x) = \gamma^1]},\] 
so the predictor $f^0_T$ satisfies its joint multicalibration condition~(\ref{eq:conditional_gamma0}) of Definition~\ref{def:jointmultical}.

Now we show that the predictor $f^1_T$ for $\Gamma^1$ satisfies its joint multicalibration condition~(\ref{eq:conditional_gamma1}). By construction, for each $\gamma^0 \in [1/m]$ the function $f^1_T$ is equal to $f^{1, \gamma^0}_T$ in the region $\{x \in \x: f^0_T(x) = \gamma^0\}$. 
Since $f^{1, \gamma^0}_T$ is output by the corresponding call to \texttt{BatchMulticalibration$^V$}, by Lemma~\ref{lem:batch_cond} this guarantees for each $G' \in \g_T^{1, \gamma^0} = \{ G \cap \{x \in \x: f^0_T(x) = \gamma^0\} : G \in \g \}$ and for each $\gamma^1 \in [1/m]$ that
$\Pr\limits_{x \in \x} [f^1_T(x) = \gamma^1, x \in G'] \left(V^1_{\gamma^0}(\gamma^1, Y_{(\gamma^1, G')})\right)^2 \leq \frac{\alpha^1}{m}$, 
implying that $\Pr_{x \in \x} [f_T(x) = (\gamma^0, \gamma^1), x \in G] \left(V^1_{\gamma^0}(\gamma^1, Y_{(\gamma^0, \gamma^1, G)})\right)^2 \leq \frac{\alpha^1}{m}$ for all $G \in \g$.

Therefore, we have for all $\gamma^0, \gamma^1, G$ the bound 
\begin{equation} \label{eq:gamma1root}
    \left|V^1_{\gamma^0}(\gamma^1, Y_{(\gamma^0, \gamma^1, G)})\right| \leq \sqrt{\frac{\alpha^1/m}{\Pr_{x \in \x} [f_T(x) = (\gamma^0, \gamma^1), x \in G]}}.
\end{equation}
But observe that we instead want to bound $
\left| V^1_{\Gamma^0 \left(Y_{(G, \gamma^0, \gamma^1)} \right)} \left(\gamma^1, Y_{ \left(G, \gamma^0, \gamma^1 \right)} \right) \right|
$, the absolute value of the \emph{true} identification function on this set. 
This is where we can make use of Assumption~\ref{assumption:conditional_interlevelset}, which gives us $L_c$-Lipschitzness of $V^1_{\gamma^0}(\cdot, \cdot)$ as a function of $\gamma^0$, as well as Assumption~\ref{assumption:antilipschitz}, which gives us $L^0_a$-anti-Lipschitzness of $V^0(\gamma^0, \cdot)$ as a function of $\gamma^0$: we obtain that
\begin{align*}
    |V^1_{\Gamma^0 \left(Y_{(G, \gamma^0, \gamma^1)} \right)}(\gamma^1, Y_{ \left(G, \gamma^0, \gamma^1 \right)})| 
    &\leq |V^1_{\Gamma^0 \left(Y_{(G, \gamma^0, \gamma^1)} \right)}(\gamma^1, Y_{ \left(G, \gamma^0, \gamma^1 \right)}) - V^1_{\gamma^0}(\gamma^1, Y_{(\gamma^0, \gamma^1, G)})| + |V^1_{\gamma^0}(\gamma^1, Y_{(\gamma^0, \gamma^1, G)})|
    \\ &\leq L_c |\gamma^0 - \Gamma^0 \left(Y_{(G, \gamma^0, \gamma^1)} \right)| + |V^1_{\gamma^0}(\gamma^1, Y_{(\gamma^0, \gamma^1, G)})|
    \\ &\leq L_c L^0_a |V^0(\gamma^0, Y_{(G, \gamma^0, \gamma^1)})| + |V^1_{\gamma^0}(\gamma^1, Y_{(\gamma^0, \gamma^1, G)})|
    \\ &\leq L_c L^0_a \sqrt{\frac{\alpha^0/m}{\Pr\limits_{x \in \x} [f_T(x) = (\gamma^0, \gamma^1), x \in G]}} + \sqrt{\frac{\alpha^1/m}{\Pr\limits_{x \in \x} [f_T(x) = (\gamma^0, \gamma^1), x \in G]}},
\end{align*}
where the fourth step is by substituting in Inequalities~\ref{eq:gamma0root} and~\ref{eq:gamma1root}.

From here, for all $\gamma^0, \gamma^1, G$ we have the bound 
\[\left|V^1_{\Gamma^0 \left(Y_{(G, \gamma^0, \gamma^1)} \right)}(\gamma^1, Y_{ \left(G, \gamma^0, \gamma^1 \right)}) \right| 
\leq \frac{(L_c L^0_a \sqrt{\alpha^0} + \sqrt{\alpha^1}) / \sqrt{m}}{\sqrt{\Pr\limits_{x \in \x} [f_T(x) = (\gamma^0, \gamma^1), x \in G]}},
\]
and after squaring both sides of the inequality, we obtain:
\[\left(V^1_{\Gamma^0 \left(Y_{(G, \gamma^0, \gamma^1)} \right)}(\gamma^1, Y_{ \left(G, \gamma^0, \gamma^1 \right)}) \right)^2 
\leq \frac{(L_c L^0_a \sqrt{\alpha^0} + \sqrt{\alpha^1})^2 / m}{\Pr\limits_{x \in \x} [f_T(x) = (\gamma^0, \gamma^1), x \in G]} \quad \text{for all } \gamma^0, \gamma^1, G.
\]
Now multiplying both sides by $\Pr\limits_{x \in \x} [f^1_T(x) = \gamma^1 | x \in G, f^0_T(x) = \gamma^0]$ and noting that 
\[(L_c L^0_a \sqrt{\alpha^0} + \sqrt{\alpha^1})^2 \leq 2((L_c L_a^0)^2 \alpha^0 + \alpha^1) = 2((L_c L_a^0)^2 \cdot 4 (L^0)^2 + 4 (L^1)^2)/m = 8 ((L^0 L^0_a L_c)^2 + (L^1)^2) / m = \alpha^1_*,\] 
we get:
\[\Pr\limits_{x \in \x} [f^1_T(x) = \gamma^1 | x \in G, f^0_T(x) = \gamma^0]
\left(V^1_{\Gamma^0 \left(Y_{(G, \gamma^0, \gamma^1)} \right)}(\gamma^1, Y_{ \left(G, \gamma^0, \gamma^1 \right)}) \right)^2 
\leq \frac{\alpha^1_* / m}{\Pr\limits_{x \in \x} [f^0_T(x) = \gamma^0, x \in G]} \quad \text{for all } \gamma^0, \gamma^1, G.
\]
For every $G \in \g$ and $\gamma^0 \in \range_{f^0_T}$, summing this inequality over all at most $m$ values $\gamma^1 \in \range_{f^1_T}$, we obtain that:
\[
\sum_{\gamma^1 \in \range_{f^1_T}} \Pr\limits_{x \in \x} [f^1_T(x) = \gamma^1 | x \in G, f^0_T(x) = \gamma^0]
\left(V^1_{\Gamma^0 \left(Y_{(G, \gamma^0, \gamma^1)} \right)}(\gamma^1, Y_{ \left(G, \gamma^0, \gamma^1 \right)}) \right)^2 
\leq \frac{\alpha^1_*}{\Pr\limits_{x \in \x} [f^0_T(x) = \gamma^0, x \in G]},
\]
so the predictor $f^1_T$ satisfies its joint multicalibration condition~(\ref{eq:conditional_gamma1}) of Definition~\ref{def:jointmultical}, thus concluding the proof.
\end{proof}



\end{document}